\documentclass[journal,twoside,web]{ieeecolor}
\usepackage{generic}
\usepackage{cite}
\usepackage{hyperref}
\hypersetup{hidelinks=true}
\usepackage{textcomp}
\usepackage{graphicx}
\usepackage{cite}
\usepackage[inkscapelatex=false]{svg}
\usepackage{cite}
\usepackage{algorithmic}
\usepackage[linesnumbered,ruled,vlined]{algorithm2e} 

\usepackage{amsmath,amssymb,amsfonts,bbm}
\usepackage{graphicx}
\usepackage{textcomp}
\usepackage{setspace}
\usepackage{booktabs}

\usepackage{epsfig}
\usepackage{times} 
\usepackage{svg}

\usepackage[nolist, nohyperlinks]{acronym}
\usepackage[table]{xcolor} 
\usepackage{makecell}

\usepackage{amsthm}

\newcommand{\reals}{\mathbb{R}}
\newcommand{\R}{\reals}

\newcommand{\Pbb}{\mathbb P}
\newcommand{\Ebb}{\mathbb E}
\newcommand{\Zbb}{\mathbb Z}

\newcommand{\Acal}{\mathcal{A}}
\newcommand{\Bcal}{\mathcal{B}}
\newcommand{\Ccal}{\mathcal{C}}

\newcommand{\Ecal}{\mathcal{E}}
\newcommand{\Fcal}{\mathcal{F}}
\newcommand{\Gcal}{\mathcal{G}}

\newcommand{\Ical}{\mathcal{I}}

\newcommand{\Kcal}{\mathcal{K}}

\newcommand{\Mcal}{\mathcal{M}}
\newcommand{\Ncal}{\mathcal{N}}

\newcommand{\Pcal}{\mathcal{P}}

\newcommand{\Rcal}{\mathcal{R}}
\newcommand{\Scal}{\mathcal{S}}

\newcommand{\Vcal}{\mathcal{V}}

\newcommand{\eqn}[1]{\begin{align} #1 \end{align}}
\newcommand{\eqnN}[1]{\begin{align*} #1 \end{align*}}
\newcommand{\seqn}[2][]{
\begin{subequations}
    #1
\begin{align} #2 \end{align}
\end{subequations}
}

\theoremstyle{plain}
\newtheorem{theorem}{Theorem}

\newtheorem{lemma}{Lemma}
\newtheorem{problem}{Problem}
\newtheorem{definition}{Definition}
\newtheorem{prop}{Proposition}

\theoremstyle{definition}
\newtheorem{assumption}{Assumption}
\newtheorem{remark}{Remark}

\theoremstyle{remark}

\makeatletter
\let\NAT@parse\undefined
\makeatother
\usepackage{hyperref}
\hypersetup{
colorlinks, 
    linkcolor={red!50!black},
    citecolor={blue!50!black},
    urlcolor={blue!80!black}
}
\usepackage{cleveref}

\crefformat{problem}{Problem~#2#1#3}
\crefformat{assumption}{Assumption~#2#1#3}
\crefformat{prop}{Proposition~#2#1#3}
\usepackage[nolist,nohyperlinks]{acronym}

\def\BibTeX{{\rm B\kern-.05em{\sc i\kern-.025em b}\kern-.08em
    T\kern-.1667em\lower.7ex\hbox{E}\kern-.125emX}}
\begin{document}
\title{Fully Byzantine-Resilient Multi-Agent Reinforcement Learning}
\author{Haejoon Lee, \IEEEmembership{Student Member, IEEE}, and Dimitra Panagou, \IEEEmembership{Senior Member, IEEE}
\thanks{*This work is supported by the Air Force Office of Scientific Research (AFOSR) under FA9550-23-1-0163}
\thanks{All authors are with the Robotics Department, University of Michigan, Ann Arbor, MI, USA
        {\tt \{haejoonl,dpanagou\}@umich.edu}}
        \thanks{$^a$Code: \href{https://github.com/joonlee16/frac-marl}{https://github.com/joonlee16/frac-marl}}
}

\maketitle
\begin{abstract}
We study distributed Byzantine-resilient actor-critic multi-agent reinforcement learning (AC-MARL), where agents collectively learn policies through local interactions. Existing methods guarantee convergence of the agents' parameters only to a neighborhood of the attack-free limit points, resulting in degraded performance. We propose Fully Resilient AC-MARL (FRAC-MARL), a decentralized method in which each agent leverages redundancy in two-hop messages to identify reliable messages. Under linear parameterizations of the value and team-reward functions and Byzantine edge attacks, where adversarial behavior is confined to the communication layer, we prove that agents' parameters converge almost surely to the same limit points as in the attack-free case over time-varying communication graphs. We introduce a novel topological condition for the convergence of our method, present a systematic method to construct such networks, and prove that this condition can be verified in polynomial time. Finally, we demonstrate our method on cooperative multi-robot formation control tasks. [code]$^a$
\end{abstract}


\section{Introduction}
\label{sec:introduction}
\IEEEPARstart{M}ulti-agent reinforcement learning (MARL) has emerged as a powerful and scalable extension of reinforcement learning (RL)~\cite{sutton1998reinforcement} for learning optimal policies of multiple agents interacting within a shared environment~\cite{zhang2021multi}. Distributed cooperative MARL in particular considers multiple agents that aim to optimize a shared global objective - often expressed as the average of local rewards - through information sharing with their neighbors in a network~\cite{zhang2018fully, lin2021multi, qu2020scalable}.

In this paper, we consider fully distributed actor-critic MARL (AC-MARL) methods studied in~\cite{zhang2018fully, zeng2022learning, dai2023distributed, dai2025distributed}, where agents with heterogeneous reward functions perform consensus-based updates of their local policy and value-function parameter estimates over a communication network. Compared with distributed tabular Q-learning approaches~\cite{kar2013QD}, these methods offer improved scalability to problems with large state and action spaces. Representative deep actor-critic MARL methods include MADDPG~\cite{lowe2017multi}, COMA~\cite{foerster2018counterfactual}, and MAPPO~\cite{yu2022surprising}. However, these methods typically rely on centralized training or centralized value estimation.

For fully distributed AC-MARL,~\cite{zhang2018fully} established convergence under linear function approximation, while finite-time convergence guarantees were provided in~\cite{zeng2022learning}. Later,~\cite{dai2023distributed} established convergence under linear function approximation for directed communication graphs. 
Finally,~\cite{dai2025distributed} presented an algorithm with nonlinear function approximation and established asymptotic convergence guarantees.

Despite these merits, distributed MARL algorithms, similarly to other distributed learning and optimization methods, are highly susceptible to adversarial attacks that corrupt or manipulate information. In the distributed systems literature, the Byzantine model represents an omniscient adversary capable of injecting arbitrary disruptions through compromised hardware, software, or communication channels~\cite{su2021byzantine, leblanc2013resilient}. As a result, a wide range of resilient algorithms have been developed to contain the impact of Byzantine agents on distributed consensus~\cite{leblanc2013resilient, yuan2025resilient, CDC2025}, optimization~\cite{sundaram2019distributed, su2021byzantine, yemini2025resilient}, learning frameworks~\cite{fang2022bridge, chen2017distributed, blanchard2017machine}, and, more recently, multi-agent LLM decision-making systems~\cite{EMNLP2026, luo2025weighted}.

Similarly, Byzantine-resilient distributed MARL has received increasing attention in recent years. Early studies demonstrated that even a single adversarial agent can significantly degrade or destabilize learning in cooperative MARL settings~\cite{figura2021adversarial, xie2021towards}. In fact, it has been shown in~\cite{hairi2024onthehardness} that learning optimal policies is generally impossible in the presence of Byzantine agents.

To address this vulnerability, several Byzantine-resilient MARL algorithms have been proposed. The distributed tabular Q-learning algorithm in~\cite{kar2013QD} was extended in~\cite{xie2023communication} by incorporating trimmed-mean aggregation to guarantee convergence in the presence of Byzantine agents. The trimmed-mean strategy was subsequently incorporated into consensus-based AC-MARL algorithms with linear function approximation to achieve resilience against Byzantine agents~\cite{wu2021byzantine_journal, yao2024communication_efficient}. More recently,~\cite{ye2024resilient, gong2026resilient} combined projection-based updates with trimmed-mean aggregation to further strengthen the resilience of AC-MARL. Similarly,~\cite{medhi2023robust} employs geometric-median aggregation to mitigate Byzantine attacks, although it does not provide convergence guarantees and considers only non-colluding adversaries.

Despite these advances, existing methods have several limitations. First, these approaches guarantee convergence only to a neighborhood of the attack-free limit point of the attack-free case. Thus, Byzantine agents may strategically manipulate the exchanged parameters to substantially degrade the performance of the non-Byzantine agents~\cite{ye2024resilient}. Although~\cite{ye2024resilient, gong2026resilient} mitigate this issue, they still guarantee only neighborhood convergence, and the size of the resulting neighborhood is generally difficult to characterize.

Furthermore, many Byzantine-resilient AC-MARL methods, including~\cite{ye2024resilient, yao2024communication_efficient, gong2026resilient}, require the communication network to satisfy $(2F+1)$-robustness~\cite{leblanc2013resilient}. Since verifying such properties is coNP-complete~\cite{zhang2015notion}, their applicability to large-scale systems may be limited. A separate line of work~\cite{lin2020toward, lin2024robust, fang2025provably} avoids these $r$-robustness requirements, but instead relies on a trusted central coordinator.

In addition, decentralized MARL with heterogeneous local rewards can be viewed through the lens of decentralized stochastic optimization under non-i.i.d. data distributions. Compared with Byzantine-resilient deterministic optimization~\cite{sundaram2019distributed, su2021byzantine} and stochastic optimization under i.i.d. assumptions~\cite{fang2022bridge, guo2022byzantine}, the non-i.i.d. setting introduces additional challenges due to heterogeneity and stochasticity across agents, which induces intrinsic bias that is further exacerbated by Byzantine attacks~\cite{wu2023byzantine_stochastic}. The work in~\cite{el2021collaborative} studied Byzantine-resilient stochastic optimization with non-i.i.d. data from a consensus perspective under a complete communication graph. In~\cite{wu2023byzantine_stochastic}, authors iteratively remove values farthest from the weighted mean to construct a doubly stochastic mixing matrix with a sufficiently small contraction factor, thereby ensuring convergence. In contrast,~\cite{he2022byzantine} proposed a clipping-based mechanism that clips received values to neighborhoods centered at the receiving agents' values. Nevertheless, these methods only guarantee convergence to a neighborhood of a stationary point of the underlying optimization problem.

To address these limitations, we build on our prior work on resilient distributed Q-learning~\cite{CDC2026} and extend the framework to decentralized Byzantine-resilient AC-MARL over time-varying communication graphs. Specifically, we propose Fully Resilient AC-MARL (FRAC-MARL), in which agents use redundant information relayed through two-hop communication to identify reliable messages before incorporating them into their local updates. Under a sufficient topological condition on the communication networks and a weaker Byzantine attack model, we establish that the agents' parameters converge almost surely to the exact limit points of the attack-free scenario, rather than to a neighborhood.

\subsubsection*{Contributions} 
Our contributions are as follows:
\begin{itemize}
\item We propose a novel decentralized Byzantine-resilient AC-MARL algorithm, termed Fully Resilient AC-MARL (FRAC-MARL), that leverages redundancy in two-hop communication to identify and incorporate reliable messages for learning over time-varying communication graphs.

\item We introduce a novel topological condition, termed $(r,r')$-redundancy, and prove that under linear function approximation it guarantees that FRAC-MARL converges almost surely to the same limit points as the attack-free case under Byzantine edge attacks, a weaker Byzantine model in which adversarial behavior is restricted to the communication layer.

\item We provide a constructive procedure for designing $(r,r')$-redundant graphs and show that this property can be verified in polynomial time, in contrast to $r$-robustness~\cite{leblanc2013resilient}, which is widely adopted in existing work~\cite{ye2024resilient, yao2024communication_efficient, gong2026resilient} but is coNP-complete to verify~\cite{zhang2015notion}.

\item We demonstrate our method can be applied even with nonlinear function approximation in a cooperative multi-agent formation task.
\end{itemize}

\section{Notation}

We denote the cardinality of a set $\Ccal$ as $|\Ccal|$. We denote the sets of non-negative and positive integers as $\mathbb Z_{\geq 0}$ and $\mathbb Z_{>0}$.  A multiset $\Ccal$ is a collection in which elements may occur with multiplicity. For a (multi)set $\Ccal$, we define ${\rm mode}(\Ccal)$ as any element with the highest number of occurrences in $\Ccal$, with ties broken uniformly at random and ${\rm mode}(\emptyset)=0$. We denote ${\rm mode\_count}(\Ccal)$ as the maximum number of occurrences of any element in a (multi)set $\Ccal$. We denote ${\rm diag}(\cdot)$ as the diagonal matrix formed by the elements of its argument. We denote the probability and expectation of a probability space $(\Omega, \Fcal, \Pbb)$ by $\Pbb(\cdot)$ and $\Ebb(\cdot)$. We denote the Kronecker product by $\otimes$. For finite sets $\Ccal_1$ and $\Ccal_2$ and functions $g_1:\Ccal_1\to\R$ and $g_2:\Ccal_1\times \Ccal_2\to\R$, we write
$[g_1(c_1),\ c_1\in\Ccal_1]^\top\in\R^{|\Ccal_1|}$ and $[g_2(c_1,c_2),\ c_1\in\Ccal_1, c_2 \in \Ccal_2]^\top\in\R^{|\Ccal_1|\cdot |\Ccal_2|}$ for the column vector of values
$g_1(c_1)$ and $g_2(c_1,c_2)$ respectively, stacked according to fixed orderings of $\Ccal_1$ and $\Ccal_1\times\Ccal_2$, respectively, used consistently throughout the paper. A sequence of random variables $\{X_t\}_{t\geq 0}$ is said to converge to a random variable $X$ almost surely (a.s.) if $\Pbb(\lim_{t\to\infty}X_t=X)=1$.

\section{Preliminaries}
\label{sec:prelim}

We consider a system of $n$ agents interacting over a simple, undirected, and time-varying communication graph $\Gcal_t=(\Vcal, \Ecal_t)$. The vertex set $\Vcal=\{1,\dots, n\}$ represents the agents, and the edge set $\Ecal_t\subseteq \Vcal \times \Vcal$ denotes a set of communication links between agents at time $t$. Since the graph is undirected, $(i,j)\in \Ecal_t$ implies $(j,i)\in \Ecal_t$. For agent $i$, its one-hop and two-hop neighbor sets at time $t$ are denoted by $
\Ncal_{i,t}=\{j\in \Vcal \mid (i,j)\in \Ecal_t\}$ and $\Ncal^{(2)}_{i,t}=\{k\in\Vcal \mid \exists\, j\in\Ncal_{i,t}\ \text{s.t.}\ k\in\Ncal_{j,t}\setminus\{i\}\}$. The extended one-hop neighbor set is $\Bcal_{i,t} = \Ncal_{i,t}\cup \{i\}$. A path of length $k\in\Zbb_{>0}$ is a sequence of vertices $(y_0, \dots, y_k)$ such that $(y_{\ell-1}, y_\ell)\in \Ecal_t$ $\forall \ell\in\{1,\dots,k\}$. A graph is connected if there exists a path between any pair of nodes.

In our setting, agents $i\in\Vcal$ not only communicate with their direct neighbors $\Ncal_{i,t}$ at time $t$, but also with their two-hop neighbors $\Ncal^{(2)}_{i,t}$ through relaying, which we refer to as \emph{two-hop communication}:
\begin{definition}[Two-hop communication]
\label{def:two_hop} Let $m^{j\to i}_{j'}(t)$ denote the copy of the message $m_{j'}(t)$ originating from any agent $j'\in\mathcal V$ and delivered to agent $i$ by agent $j$ at time $t$. We set $m^{j\to i}_{j'}(t)=\varnothing$ if the message is not received by agent $i$.

At each time $t\in\Zbb_{\ge 0}$, agents exchange messages over $\Gcal_t=(\Vcal,\Ecal_t)$
in two consecutive rounds. In round
$z=1$ (\emph{sending}), each agent $i$ transmits its own message $m_i(t)$ to every
neighbor $j\in\Ncal_{i,t}$. In round $z=2$ (\emph{relaying}), each agent $i$ forwards to every neighbor $j\in\Ncal_{i,t}$ the messages it received in round $1$, i.e., $\left\{
m_{k}^{k\to i}(t)\right\}_{ k\in\Ncal_{i,t}}$.
Thus, agent $j$ obtains a copy of the messages originating from its two-hop neighbors $\Ncal^{(2)}_{j,t}$ through an intermediate neighbor $i\in\Ncal_{j,t}$. 
\end{definition}

The environment is modeled as a networked multi-agent Markov Decision Process (MDP) defined by a tuple $(\Scal, \{\Acal^i\}_{i \in \Vcal}, P, \{r^i\}_{i \in \Vcal}, \{\Gcal_t\}_{t \geq 0},\gamma)$. Here, $\Scal$ denotes the finite state space shared by all agents, and $\Acal^i$ denotes the finite action space of agent $i$. The joint action space is given by $\Acal = \prod_{i\in\Vcal}\Acal^i$. The function $P:\Scal\times\Acal\times\Scal\rightarrow[0,1]$ specifies the state transition probability, where $P(s' \mid s,a)$ denotes the probability of transitioning to state $s'$ from state $s$ under action $a$. Each agent $i$ has a local reward function $r^i:\Scal\times\Acal\rightarrow\mathbb{R}$, and $\gamma\in(0,1)$ is the discount factor. We assume that the state and joint action are globally observable, while rewards are private and observed only by the corresponding agents.

At each time step $t$, each agent $i\in\Vcal$ observes the current state $s_t\in \Scal$ and independently selects an action
\eqn{
a_t^i\sim\pi^i(\cdot\mid s_t),
}
where $\pi^i:\Scal\times\Acal^i\rightarrow[0,1]$ is the local stochastic policy. The resulting joint action is $a_t=(a_t^1,\dots,a_t^n)\in\Acal$. We assume that agent actions are conditionally independent given the current state, that is,
\eqn{\label{eq:factorized_policy}
\pi(a\mid s)=\prod_{i\in\Vcal}\pi^i(a^i\mid s).
}

After executing $a_t^i$ at the state $s_t$, each agent $i$ receives a local realized reward $r_{t+1}^i= r^i(s_t,a_t)$,
and the environment transitions to the next state according to $
s_{t+1}\sim P(\cdot\mid s_t,a_t)$.

\subsection{Objective and Policy Parameterization} 
For a joint policy $\pi$ and a fixed initial-state distribution $\nu$ over $\Scal$,
the agents collectively maximize a globally averaged, discounted long-term
reward
\eqn{\label{eq:global_objective}
J(\pi)
& :=
\Ebb_{s_0\sim\nu,\,\pi}
\Big[\textstyle\sum_{t=0}^{\infty}\gamma^{t}\bar r(s_t,a_t)\Big],}
where $\bar r(s,a)=\frac{1}{n}\sum_{i\in\Vcal}r^i(s,a)$ represents the network-wise average reward at state-action pair $(s,a)$. We also define the associated action- and state-value functions
\eqn{
Q_\pi(s,a)
& :=
\Ebb_{\pi}
\Big[\textstyle\sum_{t=0}^{\infty}\gamma^{t}\bar r(s_t,a_t)
\;\Big|\;s_0=s,\,a_0=a\Big],\\
V_\pi(s)
& :=
\textstyle\sum_{a\in\Acal}\pi(a\mid s)\,Q_\pi(s,a).
}
Since $Q_\pi(s,a)=\bar r(s,a)+\gamma\sum_{s'\in\Scal}P(s'\mid s,a)V_\pi(s')$,
the temporal-difference (TD) error
\eqn{\label{eq:td_error}
\delta_t=\bar r(s_t,a_t)+\gamma V_\pi(s_{t+1})-V_\pi(s_t)
}
is an unbiased sample of the global advantage, i.e.,
$\Ebb[\delta_t\mid s_t=s,\,a_t=a]=Q_\pi(s,a)-V_\pi(s)$.

To enable scalable learning in large state and action spaces, we parameterize
each local policy~\cite{bhatnagar2009natural}. Let $\pi^i_{\theta^i}:\Scal\times\Acal^i\to[0,1]$ denote the local policy of agent $i\in \Vcal$ parametrized by
$\theta^i\in\Theta^i\subset\R^{b_i}$. By the conditional independence
in~\eqref{eq:factorized_policy}, we have 
\eqnN{
\pi_\theta(a\mid s)&=\prod_{i\in\Vcal}\pi^i_{\theta^i}(a^i\mid s),
\\
P_\theta(s'\mid s)&=\sum_{a\in\Acal}\pi_\theta(a\mid s)P(s'\mid s,a),
}
where $\theta=[(\theta^1)^\top,\dots,(\theta^n)^\top]^\top
\in\prod_{i\in\Vcal}\Theta^i=\Theta$. For any $\theta\in\Theta$, the Markov process $\{s_t\}_{t\in\mathbb{Z}_{\geq0}}$
induced by $\pi_\theta$ is irreducible and aperiodic. This guarantees unique, strictly positive stationary
distributions $d_\theta(s)$ and $d'_\theta(s,a)=d_\theta(s)\pi_\theta(a\mid s)$
of states and state-action pairs. 

Writing $J(\theta):=J(\pi_\theta)$,
$Q_\theta:=Q_{\pi_\theta}$, and $V_\theta:=V_{\pi_\theta}$, our objective is to
find $\theta^\star=\arg\max_{\theta\in\Theta}J(\theta)$. We follow the multi-agent
policy gradient theorem from~\cite{ye2024resilient}, where each agent improves its policy
along the direction
\eqn{\label{eq:policy_gradient}
h^i(\theta)
:= \Ebb_{d_\theta, \pi_\theta}
\left[\nabla_{\theta^i}\log\pi^i_{\theta^i}(a^i\mid s)
\big(Q_\theta(s,a)-V_\theta(s)\big)\right].
}
Note that~\eqref{eq:policy_gradient} follows the form given in~\cite{figura2021adversarial,ye2024resilient} which is the surrogate for $\nabla_{\theta^i}J(\theta)$ obtained by taking the
expectation under the on-policy stationary distribution $d_\theta$ rather than
the discounted visitation measure.
 
\subsection{Decentralized Policy Evaluation}
 
Evaluating~\eqref{eq:td_error} requires evaluations of $\bar r(s,a)$ and $V_\theta(s)$, which are not available to individual agents as rewards are private. Therefore, we adopt the consensus-based AC-MARL algorithm from~\cite{zhang2018fully}, where each agent
maintains approximations $\bar r(s,a;\lambda^i)\approx\bar r(s,a)$ and
$V(s;v^i)\approx V_\theta(s)$, parameterized by $\lambda^i$ and $v^i$, respectively, while reaching consensus on $\lambda^i,v^i$ over the communication graph
$\Gcal_t$. In this paper, we consider linear approximation, i.e.,
\eqn{\label{eq:linear_approx}
\bar r(s,a;\lambda^i)=f(s,a)^\top\lambda^i,
\qquad
V(s;v^i)=\phi(s)^\top v^i,
}
with the assumption:
\begin{assumption}
\label{assump:features}
The feature vectors $f(s, a) = [f_1(s, a), \dots, f_M(s, a)]^\top \in \R^M$ and $\phi(s) = [\phi_1(s), \dots, \phi_L(s)]^\top \in \R^L$ are uniformly bounded for any $s \in \Scal$ and $a \in \Acal$. Furthermore, if we define the feature matrix $\mathbf{F} \in \mathbb{R}^{|\Scal| \cdot |\Acal| \times M}$ with its $m$-th column $[f_m(s, a), \ s \in \Scal, a \in \Acal]^\top$ for any $m \in \{1, \dots, M\}$, and the feature matrix $\mathbf{\Phi} \in \mathbb{R}^{|\Scal| \times L}$ with $[\phi_l(s), \ s\in \Scal]^\top$ as its $l$-th column for any $l \in \{1, \dots, L\}$, then both $\mathbf{\Phi}$ and $\mathbf{F}$ have full column rank.
\end{assumption}
The assumption of the full-rank feature matrices allow us to characterize a unique asymptotically stable equilibrium in the estimations of the critic and team-averaged reward functions.

In~\cite{zhang2018fully}, at time $t\in \Zbb$, each agent first uses the current state $s_t$, next state $s_{t+1}$, and its \emph{private} reward $r^i_{t+1}$ to compute its local TD error and reward-estimation error: \seqn[\label{eq:td_and_reward_estimation}]{\psi_t^i &= r^i_{t+1} + \gamma V(s_{t+1}; v_t^i) - V(s_t; v_t^i), \\ \xi_t^i &= r^i_{t+1} - \bar{r}(s_t, a_t; \lambda_t^i),}
respectively. Then, using~\eqref{eq:td_and_reward_estimation}, each agent
updates its parameters by interleaving local stochastic approximation steps
\seqn[\label{eq:local_updates}]{
\tilde v^i_t & = v_t^i+\alpha_t^{v}\cdot\psi_t^i\cdot\nabla_{v^i}V(s_t;v_t^i),\\
\tilde\lambda_t^i & = \lambda_t^i+\alpha_t^{\lambda}\cdot\xi^i_t\cdot
\nabla_{\lambda^i}\bar r(s_t,a_t;\lambda_t^i),
}
with consensus over $\Gcal_t$:
\eqn{
\label{eq:consensus_updates}
v_{t+1}^i = \sum_{k \in \Bcal_{i,t}} w_{v,t}^{i, k} \tilde{v}^k_t, \quad   \lambda_{t+1}^i = \sum_{k \in \Bcal_{i,t}} w_{\lambda,t}^{i, k} \tilde{\lambda}^k_t,}
where $\alpha_t^{v},\alpha_t^{\lambda}>0$ are step sizes and the consensus weights satisfy $\sum_{k\in\Bcal_{i,t}}w^{i,k}_{v,t}
=\sum_{k\in\Bcal_{i,t}}w^{i,k}_{\lambda,t}=1$ for all $t\in\Zbb_{\geq0}$.

Because $\psi^i_t$ and $\xi^i_t$ are formed from $r^i_{t+1}$ rather than $\bar r(s_t,a_t)$, the local steps~\eqref{eq:local_updates} alone are generally biased with respect to the team objective. This discrepancy is compensated for by the consensus step in \eqref{eq:consensus_updates}, which drives the agents toward agreement. Consequently, every agent converges to a common limit determined by the team-average reward $\bar r$ under some assumptions.

\subsection{Byzantine Resilient AC-MARL}

However, such AC-MARL methods are vulnerable to Byzantine agents who deviate arbitrarily from the prescribed protocols~\cite{leblanc2013resilient,su2021byzantine}.
Although numerous methods have been proposed to ensure resilience~\cite{ye2024resilient, wu2021byzantine_journal, gong2026resilient}, they only guarantee convergence to a neighborhood of the attack-free limit that would be achieved in the absence of Byzantine attacks.
In fact, finding the exact optimal value functions in the presence of Byzantine agents is generally impossible~\cite{hairi2024onthehardness}. Therefore, we consider a slightly weaker but practical model in which Byzantine agents can only attack in the communication layer:

\begin{definition}[$F$-total Byzantine Edge Attack]
\label{def:byzantine_edge_attack}
Consider the two-hop communication of Definition~\ref{def:two_hop}. An edge
$(i,j)\in\Ecal_{\Bcal}^{z}(t)\subseteq\Ecal_t$ is said to be under a \textbf{Byzantine
edge attack} during round $z\in\{1,2\}$ at time $t$ if the message sent by agent $i$ is
arbitrarily altered or dropped before being received by agent $j$; that is,
\begin{align*}
z=1:&\quad m^{i\to j}_i(t)\neq m_i(t),\\
z=2:&\quad m^{i\to j}_k(t)\neq m^{k\to i}_k(t),
      \ \text{ for some } k\in\Ncal_{i,t}.
\end{align*}

 The network $\Gcal_t=(\Vcal,\Ecal_t)$ is said to be under an \textbf{$F$-total Byzantine edge attack} if
$\sum_{z=1}^{2}|\Ecal_{\Bcal}^z(t)|\le F$, $\forall t\in\Zbb_{\ge 0}$.
\end{definition}

We assume that Byzantine edge attackers have limited resources and can compromise at most $F$ communications during each time step $t$. By definition, an attacker may corrupt either an agent's original message or a message relayed by an intermediate agent. This attack model is closely related to the communication attack models studied in~\cite{gong2026resilient,lei2026distributed}.

Unlike much of the Byzantine MARL literature that assumes agents themselves are unreliable, we consider only \emph{unreliable communication.} Thus, we assume that 
\begin{assumption}
    All agents $i\in \Vcal$ are cooperative and follow the prescribed protocol. \label{assum:cooperative}
\end{assumption}

Leveraging this assumption on the agents' behavior - equivalently, that adversarial behavior is confined to the communication layer - we aim to develop AC-MARL that is fully resilient to Byzantine attacks and recovers the learning performance achievable in the absence of attacks, rather than merely guaranteeing convergence to a neighborhood of the attack-free performance. Our problem is therefore as follows:
\begin{problem}
\label{prob:problem}
Design a resilient AC-MARL algorithm such that, under
Assumptions~\ref{assump:features}-\ref{assum:cooperative} and an $F$-total
Byzantine edge attack, the critic, team-reward, and policy parameter estimates $\{v^i_t\}$, $\{\lambda^i_t\}$, and $\{\theta^i_t\}$ of
every agent $i\in \Vcal$ converge almost surely to the same limits attained in the absence of
attacks, using two-hop communications over $\Gcal_t=(\Vcal,\Ecal_t)$.
\end{problem}

By~\eqref{eq:consensus_updates}, each agent updates its parameters from the
messages received from its neighbors. Under an $F$-total Byzantine edge attack
some of these messages are corrupted, but an agent cannot tell which messages to trust and filter. The problem
therefore reduces to replacing the weighted average in \eqref{eq:consensus_updates} with a robust aggregation operator $\Rcal$ such that 
\eqn{\label{eq:general_robust_updates}
v_{t+1}^i, \lambda_{t+1}^i
=\Rcal\left(
\{m^{j\to i}_k\}_{k\in\Bcal_{i,t}\cup \Ncal_{i,t}^{(2)}}
\right)}
using only the messages available to agent $i$. Thus, the main challenge is to design an aggregation rule that can identify and filter unreliable messages without disrupting the learning performance.

We develop a method that guarantees each agent determines which messages to trust and filter through two-hop communication, such that the (i) actual induced communication remains connected and undirected~(\Cref{lem:filtered_undirected}) and (ii) $\Rcal$ completely filters out the Byzantine-induced messages in its update~(\Cref{lem:filtered_out}). Combining these two results, we demonstrate that our method ensures the same convergence guarantees as in the absence of attacks~(Theorems~\ref{thm:critic}-\ref{thm:actor}).

\section{Method}
Before presenting our method, we first provide the intuition. From a local perspective, agents cannot directly identify which messages from one-hop neighbors are compromised by Byzantine edge attacks. Therefore, many Byzantine-resilient methods rely on blind trimmed-mean techniques to discard outlier parameters. While effective, such filtering approaches have two fundamental limitations. First, because each agent independently filters the received messages, different agents may retain different subsets of their neighbors' parameters, breaking the symmetry of the underlying information flow and, consequently, the doubly stochastic property of the mixing weight matrix. Second, even after filtering, Byzantine messages that remain within the accepted range can still introduce a systematic bias, preventing the agents from fully recovering the attack-free learning behavior.

Our method leverages message redundancy through two-hop communication. Specifically, each agent relays the information received from its one-hop neighbors. As a result, agent $i\in \Vcal$ may receive information originating from the same agent $k\in\Ncal_{i,t}^{(2)}\cup\Ncal_{i,t}$ through multiple distinct one-hop neighbors, corresponding to different communication paths. Under an $F$-total Byzantine edge attack, only a limited number of messages through these paths can be compromised. As a result, the same parameter update from the same two-hop neighbor may be relayed through multiple independent paths, providing redundant copies of the information available to agent $i$. By cross-checking these redundant parameter updates, agent $i$ can aggregate messages while filtering the influence of Byzantine-corrupted information.

\subsection{Fully Resilient Actor-Critic MARL (FRAC-MARL)}

\begin{algorithm}
\SetKwInOut{Inputs}{Inputs}

\caption{Fully Resilient Actor-Critic MARL (FRAC-MARL)}
\label{alg:frac}

\Inputs{Threshold $\tau$ and step sizes ${\alpha_t^{\theta}},{\alpha_t^{v}},{\alpha_t^{\lambda}}$}

\tcp{Local Update}
Take action $a_t^i\sim\pi^i_{\theta_t^i}(\cdot \mid s_t)$, and observe next state $s_{t+1}$ and local reward $r_{t+1}^i$

Update actor
\eqnN{
\delta_t^i & = \bar{r}(s_t, a_t; \lambda^i_t) + \gamma V(s_{t+1}; v_t^i) - V(s_{t}; v_t^i)\\
\theta_{t+1}^i & = \theta_t^i +\alpha_t^{\theta} \delta_t^i \nabla_{\theta^i} \log \pi^i_{\theta^i_t}(a_t^i\mid s_t)}

Update critic and reward function
\seqn[\label{eq:local_update}]{
\psi_{t}^i & = r^i_{t+1} + \gamma V(s_{t+1}; v_t^i) - V(s_t; v_t^i) \\ 
\tilde{v}_t^i & = v_t^i+ \alpha_t^{v} \psi_{t}^i \nabla_{v^i}V(s_t;v_t^i)\\
\tilde{\lambda}_t^i & = 
\lambda_t^i + \alpha_t^{\lambda}
\left( r^i_{t+1} - \bar{r}(s_t, a_t; \lambda_t^i)\right)
\nabla_{\lambda^i}
\bar r(s_t,a_t;\lambda_t^i)
}

\tcp{Two-Hop Communication}
send $m_i(t):=(\tilde v^i_t,\tilde\lambda^i_t,i)$ to every $j\in\Ncal_{i,t}$

Relay $\{m^{k\to i}_k(t)\}_{k\in\Ncal_{i,t}}$ to every $j\in\Ncal_{i,t}$

\tcp{Redundancy-Based Filter/Consensus}

For each $k\in\Vcal\setminus\{i\}$, collect the received copies of $m_k(t)$, with at most one copy from each relaying agent $j\in\Ncal_{i,t}$, and define a multi-set
\eqnN{\Kcal^{i,k}_t\gets\big\{\,m^{j\to i}_k(t)\ \mid\ j\in\Ncal_{i,t},
\ m^{j\to i}_k(t)\neq\varnothing\,\big\}}

Collect the indices received in the two rounds,
$\Ical_{i,t}\gets\big\{k\in\Vcal\setminus\{i\}\mid \Kcal^{i,k}_t\neq\emptyset\big\}$

\For{$k\in\Ical_{i,t}$}{
Find the most repeated message,
\eqn{(\hat v_t^k,\hat\lambda_t^k,k) \gets {\rm mode}\big(\Kcal^{i,k}_t\big)
\label{eq:mode}}
}

$\Mcal_{i,t} = \left\{k \in \Vcal\setminus\{i\}\mid \mathrm{mode\_count}(\Kcal_t^{i,k}) \ge \tau\right\}$

Update the parameters
\seqn[\label{eq:averaging}]{v^i_{t+1} = w_{v,t}^{i,i}\tilde{v}_t^i + \sum_{j \in \Mcal_{i,t}} w_{v,t}^{i,j}\hat{v}_t^j, \\ \lambda^i_{t+1} =  w_{\lambda,t}^{i,i}\tilde{\lambda}_t^i + \sum_{j \in \Mcal_{i,t}} w_{\lambda,t}^{i,j}\hat{\lambda}_t^j,
} where $\sum_{j\in \Mcal_{i,t}\cup\{i\}}w^{i,j}_{v,t}=\sum_{j\in \Mcal_{i,t}\cup\{i\}}w^{i,j}_{\lambda,t}=1$ 
\end{algorithm}

The FRAC-MARL (outlined in \Cref{alg:frac}) operates by combining local actor-critic updates (lines 1-3), adopted from~\cite[Algorithm 2]{zhang2018fully}, with redundancy-based resilient aggregation (lines 4-11), which corresponds to the robust aggregator $\Rcal$ in~\eqref{eq:general_robust_updates}. At each time step $t\in \Zbb_{\geq 0}$, agent $i\in\Vcal$ first performs the standard local actor-critic updates (lines 1-3), yielding intermediate value-function and reward-model parameters $\tilde v_t^i$ and $\tilde\lambda_t^i$, respectively.


After local learning, agent $i$ broadcasts its message $m_i(t)=(\tilde v_t^i,\tilde \lambda_t^i,i)$ to its one-hop neighbors. Then each agent relays the messages it receives $\{m_k^{k\to i}(t)\}_{k \in \Ncal_{i,t}}$ to its own one-hop neighbors, allowing information to propagate over two-hop communication paths (lines 4-5). Then agent $i$ stores the messages with claimed index $k$ into the multiset $\Kcal_t^{i,k}$, one per relaying
agent, for every $k\in\Vcal\setminus\{i\}$ (line 6). Note that, an honest relay forwards at most one message per origin, so $\Kcal^{i,k}_t$ receives more than one tuple from $j$ only if $(j,i)$ is attacked. In this case, agent $i$ retains an arbitrary message for the index $k$. Then, to filter out Byzantine-influenced messages, agent $i$ performs redundancy-based filtering by constructing the set $\Mcal_{i,t}$ containing only agents $k$ whose relayed parameter updates in $\Kcal_t^{i,k}$ have a mode appearing at least $\tau$ times (lines 7-10). Finally, the accepted critic and reward-model parameters are aggregated via weighted averaging to update $v_t^i$ and $\lambda_t^i$ (line 11).

\subsection{$(r,r')$-redundancy}

Note that FRAC-MARL relies on the redundancy-based filtering (lines 7-10). The required level of redundancy, denoted by $\tau$, is a user-defined threshold that affects the filtering process and, consequently, the learning performance. In this subsection, we introduce a novel topological property that allows us to theoretically characterize the conditions under which the proposed filtering mechanism and FRAC-MARL achieve the desired performance.

First, we define an $r$-2-hop graph which is defined as below:
\begin{definition}[$r$-2-hop Graph]
Let $\Gcal_t = (\Vcal, \Ecal_t)$ be an undirected graph at time $t$. We define the \textbf{$r$-2-hop graph of $\Gcal_t$}, denoted as $\Gcal^r_t = (\Vcal, \Ecal^r_t)$, such that an edge $(i, j) \in \Ecal^r_t$ if $|\Bcal_{i,t}\cap \Ncal_{j,t}|\geq r$, where $\Bcal_{i,t}=\Ncal_{i,t}\cup\{i\}$.
\end{definition}

An illustrative example of an $r$-2-hop graph is given in~\Cref{fig:illustrative_example}. An $r$-2-hop graph contains an edge $(i,j)$ if agents $i$ and $j$ share at least $r$ neighbors (including direct links). Equivalently, there are at least $r$ vertex-disjoint paths of length at most 2 connecting them. 

\begin{definition}[$(r,r')$-redundant]
\label{def:redundant}
Let $\Gcal_t = (\Vcal, \Ecal_t)$ be an undirected graph at time $t$, and let 
$\Gcal^r_t = (\Vcal, \Ecal^r_t)$ be its $r$-2-hop graph. We say that $\Gcal_t$ is \textbf{$(r,r')$-redundant} with $r>r'\geq 0$ at time $t$ if:
\begin{enumerate}
    \item $\Gcal^r_t$ is connected, and
    \item for all $(i,j) \notin \Ecal^r_t$, $|\Bcal_{i,t}\cap \Ncal_{j,t}| \le r'$.
\end{enumerate}
\end{definition}

\begin{figure}
    \centering
\includegraphics[width=0.9\linewidth]{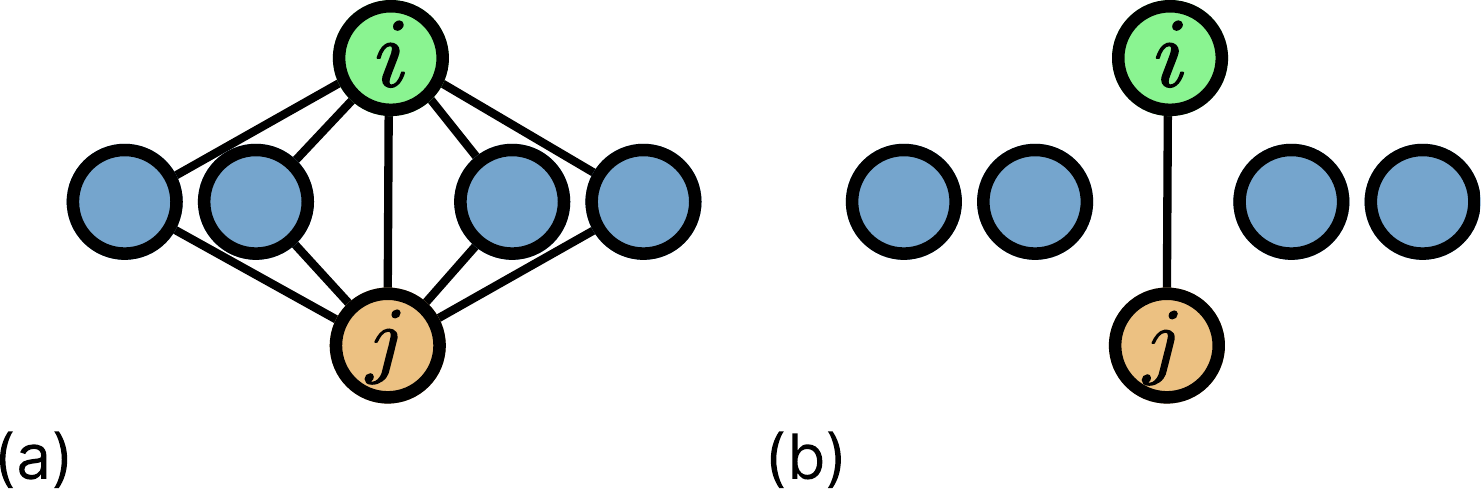}
    \caption{Visualizations of (a) graph $\Gcal_t=(\Vcal, \Ecal_t)$ and (b) its $5$-2-hop graph $\Gcal^5_t=(\Vcal, \Ecal^5_t)$ at time $t$. The edge $(i,j)\in \Ecal^5_t$, since $|\Bcal_{i,t}\cap \Ncal_{j,t}|\geq 5$.}
    \label{fig:illustrative_example}
\end{figure}
\begin{figure}
    \centering
\includegraphics[width=0.9\linewidth]{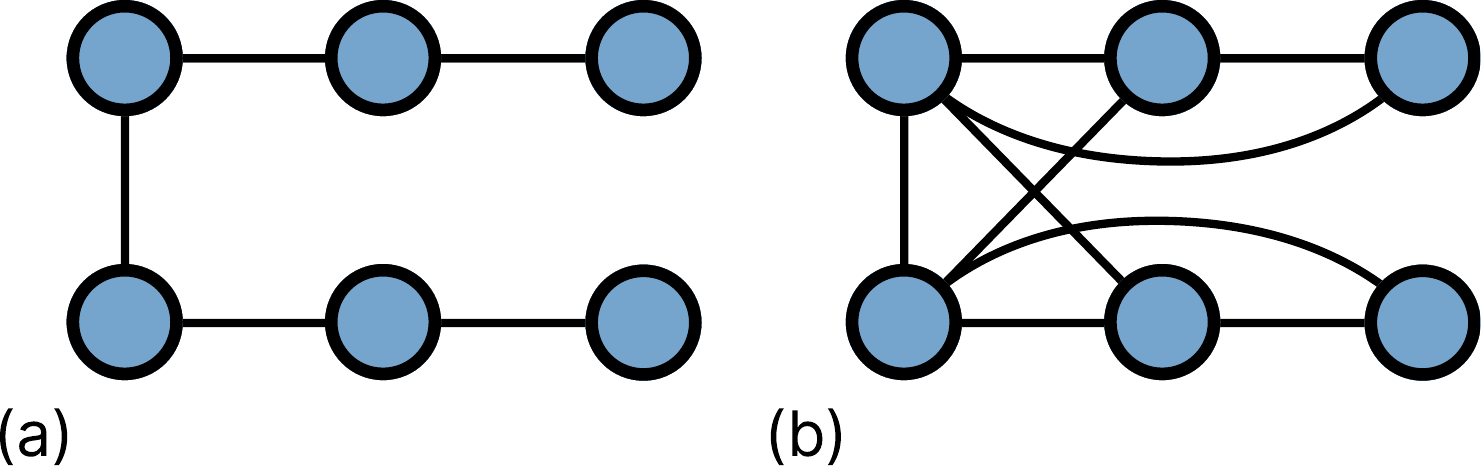}
    \caption{(a) $(1,0)$-redundant graph and (b) its 1-2-hop graph.}
    \label{fig:example_of_transformation}
\end{figure}

\Cref{fig:example_of_transformation} visualizes an $(r,r')$-redundant graph and its $r$-2-hop graph. A graph $\Gcal_t$ is $(r,r')$-redundant if two things hold. First, its $r$-2-hop graph is connected. Second, for all agent pairs not connected in the $r$-2-hop graph, they share at most $r'$ neighbors (including direct links) in the graph $\Gcal_t$. That is, for any $i,j\in \Vcal$, they share either at least $r$ or at most $r'$ neighbors (including direct links).

Going back to~\Cref{fig:example_of_transformation} as an example, for every edge $(i,j)$ in the original graph (\Cref{fig:example_of_transformation}~(a)), we have $|\Bcal_{i,t}\cap \Ncal_{j,t}|=1$. Setting $r=1$, its $r$-2-hop graph (\Cref{fig:example_of_transformation}~(b)) is connected. Because every pair $i,j\in \Vcal$ of nodes not connected in the $r$-2-hop graph satisfies $|\Bcal_{i,t}\cap \Ncal_{j,t}|=0$, the graph in~\Cref{fig:example_of_transformation}~(a) is $(1,0)$-redundant.

\begin{remark}
The $(r,r')$-redundancy condition quantifies the topological condition required for agents to have symmetric validations of correct messages amidst a bounded number of adversarial communication attacks. The threshold $r$ ensures that for any two adjacent agents in the $r$-2-hop graph, there is sufficient multi-path redundancy (at least $r$ common neighbors) to verify the relayed messages. 

Conversely, when two agents are not connected in the $r$-2-hop graph, the number of neighbors they share is bounded by $r'$. As a result, this prevents Byzantine edge attacks from manipulating relayed messages to create asymmetric validation (e.g., agent $i$ having sufficient redundancy in messages from agent $j$ while the reverse does not hold). This bound therefore is intended to ensure the symmetry of the underlying communication, which is shown later in~\Cref{lem:filtered_undirected}. ~\hfill $\bullet$
\end{remark}

Now, we provide a useful property related to $(r,r')$-redundancy:

\begin{lemma}
Let $\Gcal_t=(\Vcal, \Ecal_t)$ be $(r,r')$-redundant at time $t$. Then, its $r$-2-hop graph $\Gcal^r_t=(\Vcal, \Ecal^r_t)$ is connected and undirected at time $t$.
\label{lem:undirected}
\end{lemma}
\begin{proof}
By~\Cref{def:redundant}, $\Gcal^r_t$ is connected. Furthermore, because $\Gcal_t$ is undirected, $|\Bcal_{i,t}\cap \Ncal_{j,t}|=|\Bcal_{j,t}\cap \Ncal_{i,t}|$ for any $i,j\in \Vcal$, $i\neq j$. Hence, $\Gcal^r_t$ is also undirected.
\end{proof}

\Cref{lem:undirected} characterizes the connectivity and symmetry properties of $(r,r')$-redundancy. These results provide the foundation for the subsequent analysis of the robust aggregation mechanism presented in the next subsection.

\subsection{Theoretical Analysis}

To further facilitate the convergence analysis of our algorithm, we define an induced communication graph:

\begin{definition}[Induced Communication Graph]
Let each agent execute FRAC-MARL and construct the filtered-neighbor set $\Mcal_{i,t} = \left\{k \in \Vcal\setminus\{i\}\mid \mathrm{mode\_count}(\Kcal_t^{i,k}) \ge \tau\right\}$, where multisets $\Kcal_t^{i,k}$ are constructed in line 6 of~\Cref{alg:frac}. The \textbf{induced communication graph} is
$\Gcal_t^{\rm ind}=(\Vcal,\Ecal_t^{\rm ind})$, 
where $
(i,k)\in\Ecal_t^{\rm ind}
\iff
k\in\Mcal_{i,t}$.
\end{definition}

The induced graph $\Gcal_t^{\rm ind}$ captures the information flow where agents receive the information for update even after filtering step (lines 7-10) in FRAC-MARL. 
In general, $\Gcal_t^{\rm ind}$ is a directed graph since it is possible that
$k\in\Mcal_{i,t}$ while $i\notin\Mcal_{k,t}$.

For each $t$, we define two weight matrices
$W_{\lambda, t}=[w_{\lambda, t}^{i,k}]\in\mathbb R^{n\times n}$ and $W_{v, t}=[w_{v, t}^{i,k}]\in\mathbb R^{n\times n}$ associated with
$\Gcal_t^{\rm ind}$ such that
\eqn{
w_{\lambda,t}^{i,k}=w_{v,t}^{i,k}=
\begin{cases}
\frac 1 n & \text{if }(i,k)\in\Ecal_t^{\rm ind} \\
1- \frac {|\Mcal_{i,t}|} n& \text{if }k=i \\
0 & \text{otherwise}
\end{cases}
\label{eq:weight_matrix}
}

\begin{lemma}
Let Assumption~\ref{assum:cooperative} hold.  
Let $\Gcal_t=(\Vcal, \Ecal_t)$ be $(r,r-2F-1)$-redundant where $r>2F$, and suppose that each agent $i\in \Vcal$ runs the FRAC-MARL algorithm with $\tau=r-F$ under an $F$-total Byzantine edge attack for all $t\in \Zbb_{\geq 0}$. Then, for all $t\in\Zbb_{\geq0}$, the induced communication graph $\Gcal^{\rm ind}_t$ is equal to the $r$-2-hop graph
$\Gcal^r_t=(\Vcal,\Ecal^r_t)$ of $\Gcal_t$, i.e. $\Ecal^{\rm ind}_t=\Ecal^r_t$. Furthermore, $\Gcal^{\rm ind}_t$ is connected and undirected.
    \label{lem:filtered_undirected}
\end{lemma}
\begin{proof}

By definition, $(i,k)\in\Ecal^{\rm ind}_t$ if and only if
$k\in\Mcal_{i,t}$, so it suffices to show
\begin{equation}
  (i,k)\in\Ecal^r_t \iff k\in\Mcal_{i,t}.
  \label{eq:lem2_target}
\end{equation} We show in three parts.

\textbf{Part 1 ($\Rightarrow$):}
Suppose first that $(i,k)\in\Ecal_t^r$. Then, $|\Bcal_{i,t}\cap\Ncal_{k,t}|\geq r$. By~\Cref{assum:cooperative}, all agents $i\in \Vcal$ generate and transmit the correct parameter pairs
$(\tilde v_t^i,\tilde\lambda_t^i)$ according to~\eqref{eq:local_update} (lines 1-3 in~\Cref{alg:frac}). Since at most $F$ relayed transmissions can be corrupted under an $F$-total Byzantine edge attack, agents $i$ and $k$ receive at least $r-F$ identical copies of each other's parameter tuple. Because FRAC-MARL uses the threshold $\tau=r-F$, both agents accept one another, i.e., we have $k\in\Mcal_{i,t}$.

\textbf{Part 2 ($\Leftarrow$):} Now assume to the contrary that $k\in\Mcal_{i,t}$ and $(i,k)\notin\Ecal_t^r$. Then since (i) $(i,k)\notin\Ecal_t^r$ and (ii) $\Gcal_t$ is $(r,r-2F-1)$-redundant, $|\Bcal_{i,t}\cap\Ncal_{k,t}|\le r-2F-1$.
A tuple claiming the origin of $k$ can reach agent $i$ in only two ways. First, along an
uncorrupted path: (i) agent $k$ broadcasts to $\Ncal_{k,t}$ in the first round, and
$j\in\Ncal_{i,t}\cap\Ncal_{k,t}$ forwards it to $i$ in the second round, and (ii) $i$ receives
it directly whenever $i\in\Ncal_{k,t}$; the number of such paths is exactly
$|\Bcal_{i,t}\cap\Ncal_{k,t}|$. Second, along a compromised edge: an attacker may alter a
transmission so that it carries a tuple labeled with origin $k$, whether or not the sender
ever received one. Each compromised edge contributes at most one such tuple to
$\Kcal_t^{i,k}$, and at most $F$ edges are compromised in total. Hence
\eqnN{
|\Kcal_t^{i,k}|\le|\Bcal_{i,t}\cap\Ncal_{k,t}|+F\le(r-2F-1)+F<\tau=r-F,
}
so $\mathrm{mode\_count}(\Kcal_t^{i,k})\le|\Kcal_t^{i,k}|<\tau$ and $k\notin\Mcal_{i,t}$, which is a contradiction.

\textbf{Part 3:} Combining Parts 1 and 2, we have shown that~\eqref{eq:lem2_target} holds. Thus, $\Ecal^{\rm ind}_t=\Ecal^r_t$. Also, together with~\Cref{lem:undirected}, this implies $\Gcal^{\rm ind}_t$ is undirected and connected. 
\end{proof}

\Cref{lem:filtered_undirected} shows that the (i) communication graph induced by the FRAC-MARL filtering mechanism coincides with the $r$-2-hop graph of the underlying communication
graph $\Gcal_t$, and thus (ii) is undirected and connected at every $t$. Thus, the resulting weight matrices $W_{v,t}$ and $W_{\lambda,t}$ defined in~\eqref{eq:weight_matrix} are doubly stochastic. Also their nonzero elements are uniformly bounded below by $1/n$.

\begin{lemma}
\label{lem:filtered_out}
Let~\Cref{assum:cooperative} hold.  
Let $\Gcal_t=(\Vcal, \Ecal_t)$ be $(r,r-2F-1)$-redundant where $r>2F$, and suppose that each agent $i\in \Vcal$ runs the FRAC-MARL algorithm with $\tau=r-F$ under an $F$-total Byzantine edge attack for all $t\in \Zbb_{\geq 0}$. Then, for any time $t\in \Zbb_{\geq 0}$, agent $i \in \Vcal$, and every accepted agent indices $k\in\Mcal_{i,t}$, we have $
(\hat v_t^k,\hat\lambda_t^k)
=
(\tilde v_t^k,\tilde\lambda_t^k)$,
where $(\hat v_t^k,\hat\lambda_t^k)$ and $(\tilde v_t^k,\tilde\lambda_t^k)$ are defined in~\eqref{eq:mode} and~\eqref{eq:local_update}, respectively.
Consequently, the consensus update laws~\eqref{eq:averaging} depend only on true local
parameter estimates of agents and are independent of Byzantine-modified
transmissions.
\end{lemma}

\begin{proof}

By~\Cref{assum:cooperative}, all agents $k\in \Vcal$ generate parameter pairs
$(\tilde v_t^k,\tilde\lambda_t^k)$ according to~\eqref{eq:local_update} and transmit them. Also, since $\Gcal_t$ is $(r,r-2F-1)$-redundant and is under $F$-total Byzantine edge attack, any pair of agents connected in the $r$-2-hop graph has at least $r-F$ independent relayed copies of the correct message at each time step $t$. 

Now assume to the contrary that 
$(\hat v_t^k,\hat\lambda_t^k)\neq(\tilde v_t^k,\tilde\lambda_t^k)$. Acceptance requires at least $\tau=r-F$ identical occurrences. Since only $F$ transmissions can be corrupted in total, at most $F$ copies of any altered value can exist. However, since $r>2F$, we get $\tau = r-F > F$, which leads to a contradiction.
\end{proof}

Using these lemmas, we now provide our main results.
To establish almost sure convergence under linear function approximation with \Cref{assump:features}, we formalize the following standard regularity conditions regarding the reward bounds, learning step sizes, and parameter spaces.

\begin{assumption}
\label{assum:markov}
For every agent $i\in\Vcal$, state $s\in\Scal$, action $a^i\in\Acal^i$, and parameter $\theta^i\in\Theta^i$, the local policy $\pi_{\theta^i}^i(a^i\mid s)$ is positive and continuously differentiable with respect to $\theta^i$ over $\Theta^i$.
\end{assumption}

\begin{assumption}\label{assum:reward_bound}The local reward function $r^i(s, a)$ is uniformly bounded for every agent $i \in \Vcal$, state $s \in \Scal$, and joint action $a \in \Acal$.\end{assumption}

\begin{assumption}\label{assum:step_sizes} The stochastic approximation step sizes $\alpha_t^{v}$, $\alpha_t^{\lambda}$, and $\alpha_t^{\theta}$ are positive sequences satisfying:\eqn{\sum_{t=0}^{\infty} \alpha_t^q = \infty, \quad \sum_{t=0}^{\infty} (\alpha_t^{q})^2 < \infty, \quad \forall q \in \{v, \lambda, \theta\},}along with the two-time-scale condition $\alpha_t^\theta = o(\alpha_t^{v})=o(\alpha_t^{\lambda})$, and the limit tracking property $\lim_{t \to \infty} \alpha_{t+1}^q/\alpha_{t}^q = 1$.\end{assumption}

\begin{assumption}\label{assum:projection}For agent $i \in \Vcal$, the parameter space $\Theta^i \subset \R^{b_i}$ is compact and hyper-rectangular. The policy parameter update includes a projection operator $\Psi_{\Theta^i} : \R^{b_i} \rightarrow \Theta^i$ that maps any exterior iterate back onto the boundary of $\Theta^i$.\end{assumption}

Assumptions~\ref{assum:markov} and~\ref{assum:reward_bound} provide the regularity and boundedness conditions needed for the policy-gradient and TD-error terms to be well defined. \Cref{assum:step_sizes} establishes the two-time-scale
stochastic approximation framework, which is common in actor-critic setting~\cite{bhatnagar2009natural}. Since the actor parameters $\theta_t$ change asymptotically more slowly than the critic and reward parameters $v_t$ and $\lambda_t$, $v_t$ and $\lambda_t$ evolve on a faster timescale than $\theta_t$, allowing the critic and reward parameter estimates to track their equilibria for the current policy while the actor parameters are effectively quasi-static. Finally, \Cref{assum:projection} ensures that projection operator $\Psi_{\Theta^i}$ satisfies the conditions required by the projected stochastic approximation analysis used to establish convergence of the limiting ODEs, following the analysis of~\cite{zhang2018fully,ye2024resilient}. Using $\Psi_{\Theta^i}$, for $y\in\R^{b_i}$
and $\theta^i\in\Theta^i$, we further define
\eqn{\label{eq:tangent_projection}
\hat\Psi_{\Theta^i}[y]
:=\lim_{0<\eta\to 0}\frac{\Psi_{\Theta^i}(\theta^i+\eta y)-\theta^i}{\eta},
}
the projection of $y$ onto the tangent cone of $\Theta^i$ at $\theta^i$, which
reduces to $y$ whenever $\theta^i$ lies in the interior of $\Theta^i$. When the limit  in~\eqref{eq:tangent_projection} is not unique, we define $\hat\Psi_{\Theta^i}[y]$ as the set of all possible limit points.

We first show the convergence of the critic and team reward approximation functions. Let $\mathbf{D}^{\theta}_{s} = {\rm diag}\left(\begin{bmatrix}
    d_{\theta}(s), \ s\in \Scal
\end{bmatrix}\right) \in \R^{|\Scal|\times |\Scal|}$ and $\mathbf{D}^{\theta}_{s,a} = {\rm diag}\left(\begin{bmatrix}
    d_{\theta}'(s,a), \ s\in \Scal, a \in \Acal
\end{bmatrix}\right) \in \R^{|\Scal|\cdot|\Acal|\times |\Scal|\cdot|\Acal|}$ denote the diagonal matrices corresponding to the state distribution and state-action distribution induced by policy $\pi_\theta$, respectively. Also, we denote $\mathbf{R}=\begin{bmatrix}
    \bar r(s, a), \ s\in \Scal, a \in \Acal
\end{bmatrix}^\top \in \R^{|\Scal|\cdot|\Acal|}$ and $\mathbf{R}_\theta=\begin{bmatrix}
    \sum_{a \in \Acal} \pi_{\theta}(a\mid s)\bar r(s, a), \ s\in \Scal
\end{bmatrix}^\top \in \R^{|\Scal|}$. We denote $\mathbf{P}_\theta = \begin{bmatrix}
    P_\theta(s'\mid s), s\in \Scal, s'\in \Scal
\end{bmatrix}\in \R^{|\Scal|\times |\Scal|}$.

\begin{theorem}
\label{thm:critic}
\label{thm:actor}
    Let Assumptions~\ref{assump:features}-\ref{assum:step_sizes} hold and fix
    $\theta\in\Theta$.
    Suppose that, for every $t\in\mathbb{Z}_{\geq 0}$, (i) the graph
    $\Gcal_t=(\Vcal,\Ecal_t)$ is $(r,r-2F-1)$-redundant where $r>2F$, and
    (ii) every agent $i\in\Vcal$ selects $a^i_t\sim\pi^i_{\theta^i}(\cdot\mid s_t)$
    and updates $\{v^i_t\}$ and $\{\lambda^i_t\}$ via~\eqref{eq:local_update} and~\eqref{eq:averaging} of FRAC-MARL with threshold $\tau=r-F$ and
    consensus weights~\eqref{eq:weight_matrix}, under an $F$-total Byzantine edge attack.
    Then
\begin{equation}
    \lim_{t\to\infty} v_t^i = v_\theta,
    \qquad
    \lim_{t\to\infty} \lambda_t^i = \lambda_\theta,
    \quad \text{for all } i\in\Vcal,
\end{equation}
almost surely, where $\lambda_\theta$ and $v_\theta$ are the unique 
solutions to
\seqn[\label{eq:equilibria_space}]{
    \mathbf{F}^\top \mathbf{D}_{s,a}^{\theta}
    (\mathbf{R}-\mathbf{F}\lambda_\theta)&=0,
    \label{eq:lambda}
    \\
    \mathbf{\Phi}^\top \mathbf{D}_s^\theta
    (\mathbf{R}_\theta+\gamma \mathbf{P}_\theta\mathbf{\Phi} v_\theta-\mathbf{\Phi} v_\theta)&=0.
    \label{eq:value}
}
\end{theorem}

\begin{proof}
By~\Cref{lem:filtered_out}, all Byzantine-modified transmissions are
removed by FRAC-MARL before the consensus update. Therefore, the critic
and reward parameter updates are equivalent to those of a standard
decentralized actor-critic algorithm operating over the induced
communication graph $\Gcal_t^{\rm ind}$.

Furthermore, by~\Cref{lem:filtered_undirected}, the induced graph is
connected and undirected. This means the corresponding weight matrices $W_{\lambda,t}$ and $W_{v,t}$ as defined in~\eqref{eq:weight_matrix} are
doubly stochastic by construction. Because the weight matrices are doubly stochastic with uniformly positive diagonal entries and satisfy the required connectivity condition, we get spectral radii
$\sup_{t}\rho\big(W_{q,t}(I-\tfrac1n\mathbf 1\mathbf 1^\top)W_{q,t}\big)<1$ for $q\in\{v,\lambda\}$. Furthermore, $\Gcal_t$ is
independent of the sample path $\{(s_{t'},a_{t'},r_{t'+1})\}_{t'\ge0}$ and
the iterates $\{v^i_{t'},\lambda^i_{t'},\theta^i_{t'}\}_{{t'}\ge0,\,i\in\Vcal}$. Then because $W_{\lambda,t}$ and $W_{v,t}$ are determined by $\Gcal_t$ alone, $W_{\lambda,t}$ and $W_{v,t}$
satisfy~\cite[Assumption~4]{figura2021adversarial}.

We first stack the parameters into 
\eqnN{
\lambda_t=
\begin{bmatrix}
(\lambda_t^1)^\top &
\cdots &
(\lambda_t^n)^\top
\end{bmatrix}^\top, \ v_t=
\begin{bmatrix}
(v_t^1)^\top &
\cdots &
(v_t^n)^\top
\end{bmatrix}^\top.
}
Then the update laws for both parameters can be compactly written as
\seqn[\label{eq:recursion}]{
\lambda_{t+1}
=
(W_{\lambda,t}\otimes I)
\left(
\lambda_t+
\alpha_t^\lambda
(A_t^\lambda\lambda_t+b_t^\lambda)
\right) \\
v_{t+1}
=
(W_{v,t}\otimes I)
\left(
v_t+
\alpha_t^v
(A_t^vv_t+b_t^v)
\right),
}
where
\eqnN{
A_t^\lambda
& =
-I\otimes
f(s_t,a_t)f(s_t,a_t)^\top,\\
A_t^v
& =
I\otimes
\phi(s_t)
(\gamma\phi(s_{t+1})-\phi(s_t))^\top, 
}
and
\eqnN{
b_t^\lambda=
\begin{bmatrix}
f(s_t,a_t)r_{t+1}^1\\
\vdots\\
f(s_t,a_t)r_{t+1}^n
\end{bmatrix}, \ b_t^v=
\begin{bmatrix}
\phi(s_t)r_{t+1}^1\\
\vdots\\
\phi(s_t)r_{t+1}^n
\end{bmatrix}.
}

The compact form~\eqref{eq:recursion} is mathematically in the same form as those defined in~\cite[Lemmas~1 and 3]{figura2021adversarial}. By \cite[Lemma~1]{figura2021adversarial}, $\sup_t\|\lambda_t\|<\infty$ and $\sup_t\|v_t\|<\infty$
almost surely, and by \cite[Lemma~3]{figura2021adversarial},
$\lim_{t\to\infty}\|\lambda_t-\mathbf{1}\otimes\bar{\lambda}_t\|
=\lim_{t\to\infty}\|v_t-\mathbf{1}\otimes\bar{v}_t\|=0$ almost surely, where
\eqnN{
\bar{\lambda}_t
=
\frac1n(\mathbf1^\top\otimes I)\lambda_t,
\qquad
\bar{v}_t
=
\frac1n(\mathbf1^\top\otimes I)v_t.
}
Premultiplying~\eqref{eq:recursion} by $\tfrac1n(\mathbf1^\top\otimes I)$ and using
$\mathbf1^\top W_{\lambda,t}=\mathbf1^\top W_{v,t}=\mathbf1^\top$ yields the closed
recursions
\eqnN{
\bar{\lambda}_{t+1}
& =
\bar{\lambda}_t
+
\alpha_t^\lambda
\left(
-f(s_t,a_t)f(s_t,a_t)^\top\bar{\lambda}_t
+
f(s_t,a_t)\bar r_{t+1}
\right),\\
\bar{v}_{t+1}
& =
\bar{v}_t
+
\alpha_t^v
\left(
\phi(s_t)(\gamma\phi(s_{t+1})-\phi(s_t))^\top\bar{v}_t
+
\phi(s_t)\bar r_{t+1}
\right),
}
where $\bar r_{t+1}=\frac1n\sum_{i\in\Vcal}r_{t+1}^i$ is the team-averaged realized
reward. By~\cite[Appendix B.4 Step 2]{zhang2018fully}, under Assumptions~\ref{assump:features},~\ref{assum:markov},~\ref{assum:reward_bound},~\ref{assum:step_sizes}, and that since $\{(s_t,a_t)\}$ is irreducible and aperiodic with
stationary distribution $d_\theta'$, the asymptotic behaviors of $\bar{\lambda}_t$ and $\bar{v}_t$ are described by
\seqn[\label{eq:odes}]{
\dot\lambda
& =
-\mathbf F^\top \mathbf{D}_{s,a}^{\theta}\mathbf F\lambda
+
\mathbf F^\top \mathbf{D}_{s,a}^{\theta}\mathbf{R},\\
\dot v
& =
\mathbf\Phi^\top \mathbf{D}_s^\theta
(\gamma \mathbf{P}_\theta-I)\mathbf\Phi v
+
\mathbf\Phi^\top \mathbf{D}_s^\theta\mathbf{R}_\theta.
}
Since
$-\mathbf F^\top \mathbf{D}_{s,a}^{\theta}\mathbf F$
and
$\mathbf\Phi^\top \mathbf{D}_s^\theta(\gamma \mathbf{P}_\theta-I)\mathbf\Phi$
are Hurwitz under
Assumptions~\ref{assump:features} and~\ref{assum:markov},
the ODEs~\eqref{eq:odes} admit the unique globally asymptotically stable
equilibria $\lambda_\theta$ and $v_\theta$, respectively (and thus solutions to~\eqref{eq:equilibria_space}). Combined with $\sup_t\|\lambda_t\|<\infty$ and
$\sup_t\|v_t\|<\infty$, this gives
$\bar{\lambda}_t\to\lambda_\theta$ and $\bar{v}_t\to v_\theta$ almost
surely. Therefore, since
$\|\lambda_t^i-\bar{\lambda}_t\|\le\|\lambda_t-\mathbf1\otimes\bar{\lambda}_t\|$,
\eqnN{
\|\lambda_t^i-\lambda_\theta\|
\le
\|\lambda_t^i-\bar{\lambda}_t\|
+
\|\bar{\lambda}_t-\lambda_\theta\|
\rightarrow0,
}
and similarly,
\eqnN{
\|v_t^i-v_\theta\|
\le
\|v_t^i-\bar{v}_t\|
+
\|\bar{v}_t-v_\theta\|
\rightarrow0
}
almost surely. Thus, $
\lim_{t\to\infty}\lambda_t^i=\lambda_\theta$ and $\lim_{t\to\infty}v_t^i=v_\theta$, 
almost surely for every $i\in\Vcal$, completing the proof.
\end{proof}

\Cref{thm:critic} depends on Assumptions~\ref{assump:features}-\ref{assum:step_sizes}. Except for~\Cref{assum:cooperative}, all other assumptions are standard in reinforcement learning and can be applied to a broad class of multi-agent systems, including multi-robot task allocation, navigation, and formation control, where agents employ stochastic policies with bounded rewards over a prescribed operating domain. We discuss the implications of~\Cref{assum:cooperative} in more detail later in this section.

\Cref{thm:critic} establishes that, despite Byzantine edge attacks, the proposed algorithm enables all agents to reach consensus on the critic and team-averaged reward approximations almost surely. Specifically, the parameter estimates converge to $\lambda_\theta$ and $v_\theta$, which are, respectively, the least-squares approximation of the global reward function $\bar r$ and the unique solution to the mean square projected Bellman equation under the linear function approximation.

Having established convergence of the critic and team-reward parameter estimates, we next discuss the convergence of the actor in a slower timescale:

\begin{theorem}
\label{thm:actor}
Let Assumptions~\ref{assump:features}-\ref{assum:projection} hold. Suppose that, for every $t\in\mathbb{Z}_{\geq 0}$, (i) the graph $\Gcal_t=(\Vcal,\Ecal_t)$ is $(r,r-2F-1)$-redundant where $r>2F$, and (ii) every agent $i\in\Vcal$ executes the FRAC-MARL algorithm with threshold $\tau=r-F$ and consensus weights~\eqref{eq:weight_matrix}, under an $F$-total Byzantine edge attack. Then, for every agent $i\in\Vcal$, the policy parameter $\theta^i_t$ converges almost surely to a point in the set of locally asymptotically stable equilibria of the ordinary differential equation
\eqn{\label{eq:actor_odes}
\dot{\theta}^i=\hat{\Psi}_{\Theta^i}\left[\mathbb{E}_{d_\theta, \pi_\theta, P}
\left[\delta_\theta(s_t, a_t, s_{t+1})\nabla_{\theta^i}
\log \pi^i_{\theta^i}(a^i_t \mid s_t)\right]\right],
}
with
\eqnN{\delta_{\theta}(s_t,a_t,s_{t+1})=\bar{r}(s_t,a_t;\lambda_{\theta})+\gamma V(s_{t+1};v_\theta)-V(s_t;v_{\theta}),}
where parameters $\lambda_{\theta}$ and $v_{\theta}$ are the globally asymptotically stable equilibria under the global policy $\pi_\theta$.
\end{theorem}
\begin{proof}
The proof proceeds by combining the convergence result established in
\Cref{thm:critic} with the analysis developed
in~\cite[Theorem~4.10]{zhang2018fully}
and~\cite[Theorem~6]{ye2024resilient}. The actor update of each agent can be written as
\eqn{
\label{eq:first_actor_update}
\theta_{t+1}^i
=
\Psi_{\Theta^i}
\left(
\theta_t^i+
\alpha_t^\theta \cdot
\delta_t^i  \cdot
\nabla_{\theta^i}
\log \pi^i_{\theta^i_t}(a^i_t \mid s_t)
\right),
}
where
\eqn{\delta_t^i & = \bar{r}(s_t, a_t; \lambda^i_t) + \gamma V(s_{t+1}; v_t^i) - V(s_{t}; v_t^i).
}

By~\Cref{thm:critic}, the critic and reward parameters $(v_t^i,\lambda_t^i)$ almost surely converge to $(v_\theta,\lambda_\theta)$ for every fixed $\theta\in \Theta$. Also, the map $\theta\mapsto(v_\theta,\lambda_\theta)$ is continuous on $\Theta$ by~\Cref{assum:markov} and the continuity of $d_\theta$ in $\theta$~\cite[Appendix B.3]{zhang2018fully}. Since the actor evolves on a slower timescale by~\Cref{assum:step_sizes}, $\theta_t$ can be held constant when analyzing the faster recursions, and the standard two-time-scale argument from~\cite[Theorem 4.10]{zhang2018fully} and~\cite[Theorem 6]{ye2024resilient} applies, yielding 
\eqn{\label{eq:tracking}
\lim_{t\to\infty}\|v^i_t-v_{\theta_t}\|
=\lim_{t\to\infty}\|\lambda^i_t-\lambda_{\theta_t}\|=0
\quad\text{a.s.},\ \forall i\in\Vcal .
}
Hence the bias $\varepsilon^i_t:=\delta^i_t-\delta_{\theta_t}(s_t,a_t,s_{t+1})$
introduced by using $(v^i_t,\lambda^i_t)$ in place of
$(v_{\theta_t},\lambda_{\theta_t})$ satisfies $\varepsilon^i_t\to0$ almost
surely by~\eqref{eq:tracking}.

Combining this with Assumptions~\ref{assump:features} and~\ref{assum:reward_bound}, we know that $\varepsilon_t^i$ in the actor TD error $\delta_t^i=\delta_{\theta_t}(s_t, a_t,s_{t+1})+\varepsilon_t^i$ is bounded and $\varepsilon_t^i\to 0$ almost surely, and the actor update asymptotically becomes
\eqn{\label{eq:second_actor_update}
\theta_{t+1}^i
=
\Psi_{\Theta^i}
\left(
\theta_t^i+
\alpha_t^\theta \cdot
\delta_{\theta_t}(s_t, a_t,s_{t+1})
\nabla_{\theta^i}
\log \pi^i_{\theta^i_t}(a^i_t \mid s_t)
\right).
}

At this point, the actor recursion has the same limiting stochastic approximation form as those analyzed in~\cite[Theorem~4.10]{zhang2018fully} and~\cite[Theorem~6]{ye2024resilient}. It now remains to verify that the assumptions required for their stochastic approximation analysis hold in our setting. More specifically, we have
\begin{itemize}
\item $\sup_t\Ebb\big(\|\delta^i_t\nabla_{\theta^i}\log\pi^i_{\theta^i_t}(a^i_t|s_t)\|
      \mid \theta_{t'},\,t'\leq t\big)<\infty$ by Assumptions~\ref{assump:features},~\ref{assum:markov},~\ref{assum:reward_bound} and~\ref{assum:projection} with~\Cref{thm:critic};
\item compact and hyper-rectangular projection set $\Theta^i$ by \Cref{assum:projection};
\item $\sum_{t=0}^{\infty} \alpha_t^\theta = \infty$, $\sum_{t=0}^{\infty} (\alpha_t^{\theta})^2 < \infty$, and $\lim_{t \to \infty} \alpha_{t+1}^{\theta}/\alpha_{t}^{\theta} = 1$ by~\Cref{assum:step_sizes}; 
\item bias $\varepsilon^i_t\to0$ almost surely (which is established above);
\item continuous limiting mean update $g(\theta^i_t)=\Ebb_{d_\theta,\pi_\theta,P}\big[\delta_\theta(s_t,a_t,s_{t+1})\nabla_{\theta^i}\log\pi^i_{\theta^i_t}(a^i_t|s_t)\big]$ by~\Cref{assum:markov};
\item continuity of $\theta\mapsto(v_\theta,\lambda_\theta)$ (which is established above).
\end{itemize}

Hence, by~\cite[Theorem 6]{ye2024resilient}, the asymptotic behavior of the actor is governed by
\eqnN{
\dot{\theta}^i=\hat{\Psi}_{\Theta^i}
\left[
\mathbb{E}_{d_\theta,\pi_\theta, P}
\left[
\delta_{\theta}(s_t,a_t,s_{t+1})
\nabla_{\theta^i}
\log\pi^i_{\theta^i}(a_t^i|s_t)
\right]
\right],
}
and $\theta^i$ converges almost surely to a point in the set of locally asymptotically stable equilibria of~\eqref{eq:actor_odes}.
\end{proof}



Our analysis is in the same spirit as those developed in~\cite{zhang2018fully, figura2021adversarial, ye2024resilient}. 
The convergence guarantees established in~\Cref{thm:actor} show that every agent converges almost
surely to an asymptotically stable equilibrium of~\eqref{eq:actor_odes}, which is the standard convergence guarantee for actor-critic algorithms even in the single-agent setting~\cite{bhatnagar2009natural,zhang2018fully}. Our result nevertheless represents a notable improvement over existing Byzantine-resilient MARL methods (e.g.,~\cite{wu2021byzantine,ye2024resilient,gong2026resilient}), which guarantee convergence only to \emph{a neighborhood}.

The key distinction is that FRAC-MARL completely removes the Byzantine-induced bias that affects the existing work before it enters the consensus updates. Specifically, Byzantine attacks introduce bias in two ways: (i) corrupted messages that, if accepted, directly perturb the consensus updates, and (ii) asymmetric information flow created by filtering, which itself introduces consensus bias. As established by
Lemmas~\ref{lem:filtered_undirected}-\ref{lem:filtered_out}, our method removes such biases through $(r,r')$-redundancy. Consequently, the critic and reward estimates converge to the same values as in the Byzantine-free setting.

We note that such stronger convergence guarantees are obtained under~\Cref{assum:cooperative}, where all agents are assumed to follow the prescribed protocol, and Byzantine attacks are introduced only through communication corruption. This setting constitutes a weaker attack model than the classical Byzantine-agent model~\cite{su2021byzantine, leblanc2013resilient}, in which Byzantine agents may \emph{arbitrarily} deviate from the prescribed protocol.

However, this weakened attack model at the same time allows our method to remain completely independent of how the Byzantine attack unfolds, as long as the number of corrupted communications stays bounded. This is also an improvement over some of the existing work on Byzantine-resilient AC-MARL~\cite{ye2024resilient,gong2026resilient}, as they assume that Byzantine agents' policies converge to some stationary policy, which ensures stationary MDP from the perspective of non-Byzantine agents. This effectively restricts the learning-dynamics of the Byzantine attackers in the asymptotic sense. FRAC-MARL requires no restriction on the temporal or learning-dynamics behavior of Byzantine attacks. As a result, as long as the attack occurs at the communication level, our guarantee holds regardless of the behavior of the Byzantine attack.

\section{Redundant Network Graph}
 Through Theorems~\ref{thm:critic}-\ref{thm:actor}, we have shown that $(r,r-2F-1)$-redundancy where $r>2F$ plays a pivotal role in achieving resilience against the $F$-total Byzantine edge attacks. In this section, we explore different aspects of the notion of $(r,r')$-redundancy by providing (i) a systematic construction of an $(r,r')$-redundant graph (\Cref{prop:construction}) and (ii) its computation time (\Cref{prop:computation}). As the underlying structural requirements are independent of the learning dynamics, we focus on time-invariant graphs and drop the argument $t$ on a graph throughout this discussion.

We first present a systematic method to construct $(r,r')$-redundant graphs for any $r$ and $r'$:

\begin{prop}
\label{prop:construction}
    Let $\Vcal=\{1,\dots, n\}$ where $n>r$, and $\Vcal_c=\{1,\dots, r\}\subset \Vcal$. Then, a graph $\Gcal=(\Vcal, \Ecal)$ is $(r,r')$-redundant for $r>r'\geq 0$ if (i) every node in $\Vcal_c$ is connected to every other node in $\Vcal_c$ and (ii) every node $i\in\Vcal\setminus \Vcal_c$ is connected to all $r$ nodes in $\Vcal_c$.
\end{prop}
\begin{proof}
By (i) and (ii), every $i\in\Vcal_c$ is adjacent to every other node of
$\Vcal$, so $\Ncal_i=\Vcal\setminus\{i\}$ and $\Bcal_i=\Vcal$; and every
$i\in\Vcal\setminus\Vcal_c$ is adjacent to all of $\Vcal_c$, so
$\Ncal_i\supseteq\Vcal_c$ and $\Bcal_i\supseteq\Vcal_c\cup\{i\}$.

We show $|\Bcal_i\cap \Ncal_j|\geq r$ for every $i\neq j$, considering four cases.
(a) If $i,j\in\Vcal_c$, then $\Bcal_i\cap\Ncal_j=\Vcal\setminus\{j\}$, so
$|\Bcal_i\cap\Ncal_j|=n-1\geq r$ since $n>r$.
(b) If $i\in\Vcal_c$ and $j\notin\Vcal_c$, then
$\Bcal_i\cap\Ncal_j=\Ncal_j\supseteq\Vcal_c$, so $|\Bcal_i\cap\Ncal_j|\geq r$.
(c) If $i\notin\Vcal_c$ and $j\in\Vcal_c$, then
$|\Bcal_i\cap\Ncal_j|\geq|(\Vcal_c\cup\{i\})\setminus\{j\}|$. Since $j\in\Vcal_c$ and $i\notin\Vcal_c$,
this set has $(r-1)+1=r$ elements.
(d) If $i,j\notin\Vcal_c$, then $\Bcal_i\cap\Ncal_j\supseteq\Vcal_c$, so
$|\Bcal_i\cap\Ncal_j|\geq r$.

In every case $|\Bcal_i\cap\Ncal_j|\geq r$, so $(i,j)\in\Ecal^r$ for all
$i\neq j$. Hence $\Gcal$ is $(r,r')$-redundant for
every $r'$ with $r>r'\geq 0$.
\end{proof}

While \Cref{prop:construction} provides a method to construct $(r,r')$-redundant graphs, the following result establishes that we can verify the redundancy of an arbitrary graph efficiently.

\begin{prop}
Given an $r, r'\in \Zbb_{\geq 0}$ and a communication graph $\Gcal=(\Vcal, \Ecal)$ with $|\Vcal|=n$, one can verify whether $\Gcal$ is $(r,r')$-redundant in $O(n^{3})$.
    \label{prop:computation}
\end{prop}
\begin{proof}
    Let $A$ be an adjacency matrix of $\Gcal$. Then, $\bar{A}:=A^2+A$ will contain elements $\bar{a}_{ij}$ that counts the number of shared neighbors between nodes $i$ and $j$ (including node $i$ itself) i.e., $|\Bcal_i\cap \Ncal_j|$. Computing $\bar{A}$ using standard matrix multiplication requires $O(n^3)$ operations~\cite[Sec. 4]{cormen2009intro_to_alg}, and checking whether $\bar{a}_{ij}\geq r$ or $\bar{a}_{ij}\leq r'$ adds at most $O(n^2)$ operations. Next, to verify that the $r$-2-hop graph $\Gcal^r = (\Vcal, \Ecal^r)$ of $\Gcal$ is connected, one can perform a Breadth-First Search (BFS), which requires $O(n + m_r)$ time, where $m_r = |\Ecal^r|$~\cite[Sec. 22.2]{cormen2009intro_to_alg}. Therefore, the total required computation is $O(n^{3}+m_r)$. Since $m_r\leq \binom{n}{2}=O(n^2)$, $O(n^{3}+m_r)=O(n^{3})$.
\end{proof}

\Cref{prop:computation} establishes that $(r,r')$-redundancy can be verified efficiently. Compare this with $r$-robustness~\cite{leblanc2013resilient}, whose definition is given below: 
\begin{definition}[$\mathbf r$\textbf{-robustness}~\cite{leblanc2013resilient}]
    A graph $\Gcal = (\Vcal,\Ecal)$ is $\mathbf r$\textbf{-robust} if for every pair of nonempty, disjoint subsets $\Pcal_1,\Pcal_2 \subset \Vcal$, at least one of the subsets contains a node with at least $r$ neighbors outside the subset. That is, there exists a node $i \in \Pcal_k$ such that $|\Ncal_i \setminus \Pcal_k| \geq r$ for some $k \in \{1, 2\}$.
    \label{def:r_robust}
\end{definition}

While $(2F+1)$-robustness provides a sufficient condition for many Byzantine-resilient AC-MARL frameworks~\cite{ye2024resilient, xie2023communication,yao2024communication_efficient, gong2026resilient}, determining whether a graph satisfies this property is coNP-complete~\cite{zhang2015notion}. Consequently, there is no known efficient algorithm for verifying $r$-robustness in general, making robustness-based design impractical for large-scale and dynamic networks. In contrast, $(r,r')$-redundancy offers a tractable alternative, making our method more suitable.

\begin{lemma}
    Let $\Gcal$ be an $(r,r')$-redundant graph constructed according to~\Cref{prop:construction} with $n\geq2r-1$. Then, $\Gcal$ is $r$-robust.
    \label{lem:connection}
\end{lemma}

\begin{proof}
By~\Cref{prop:construction}, each $i\in\Vcal_c$ satisfies
$\Ncal_i=\Vcal\setminus\{i\}$, and each $i\in\Vcal\setminus\Vcal_c$ satisfies
$\Ncal_i\supseteq\Vcal_c$. Let $\Pcal_1,\Pcal_2\subset\Vcal$ be nonempty and
disjoint; since $|\Pcal_1|+|\Pcal_2|\leq n$, assume without loss of generality
$|\Pcal_1|\leq\lfloor n/2\rfloor$. If $\Pcal_1\cap\Vcal_c\neq\emptyset$, any
$i\in\Pcal_1\cap\Vcal_c$ gives
$|\Ncal_i\setminus\Pcal_1|=n-|\Pcal_1|\geq\lceil n/2\rceil\geq r$. Otherwise
$\Vcal_c\cap\Pcal_1=\emptyset$, so any $i\in\Pcal_1$ gives
$\Ncal_i\setminus\Pcal_1\supseteq\Vcal_c$ and thus
$|\Ncal_i\setminus\Pcal_1|\geq r$. In either case $\Pcal_1$ contains a node $i$ such that $|\Ncal_i\setminus \Pcal_1|\geq r$, completing the proof.
\end{proof}

\Cref{lem:connection} connects $(r,r')$-redundancy of graphs constructed
via~\Cref{prop:construction} to the notion of
$r$-robustness. Note that this result does not establish a general
characterization between the two properties; it applies only to the specific
class of graphs. Establishing a full characterization of the relationship between these two topological conditions remains future work.

\section{Simulation Results}

We evaluate the proposed algorithm on a cooperative formation task
built on the Multi-Particle
Environments~2 (MPE2)~\cite{lowe2017multi}. We consider a team of
$n = 10$ agents that must arrange themselves into a circular formation of
radius $R_{\rm circle} = 0.5$ around a stationary landmark located at $p^{\mathrm{lm}}\in \R^2$. Each agent estimates its actor, critic, and team-average reward functions using neural networks with a single hidden layer of 30 units and Leaky ReLU activation functions with negative slope $0.1$.

\noindent\textbf{State Space:}
At each time $t$, agent $i\in \Vcal=\{1,\ldots,10\}$ observes the state
\eqn{s_t = \big[(o_t^1)^\top, (o_t^2)^\top, \dots, (o_t^n)^\top \big]^\top \in
    \mathbb{R}^{6n},}
where $o_t^i = \big[(\dot{p}_t^i)^\top,\; (p_t^i)^\top, (p_t^{i,{\rm rel}})^\top\big]^\top$ contains the agent $i$'s velocity
$\dot{p}_t^i$, position $p_t^i \in \mathbb{R}^2$,
and its relative position to the landmark, $p_t^{i,{\rm rel}}=p_t^i - p^{\mathrm{lm}}$.

\noindent\textbf{Action Space:}
Each agent selects a discrete action
$a_t^i \in \Acal^i = \{{\rm stay},{\rm left},{\rm right},{\rm down},{\rm up}\}$, corresponding to stay still or a unit force applied along one
of the four cardinal directions.

\noindent\textbf{Reward Space:} Each agent $i$ receives a reward $r_{t+1}^i=-\operatorname{clip}(\|p_t^i-g_t^i\|_2,0,2)$ based on its distance to an assigned formation goal $g_t^i$. At each time step, $n$ goal positions $\{g_t^i\}_{i\in\Vcal}$ are evenly spaced on a circle of radius $R_{\rm circle}$ around the landmark, with the orientation determined by the current agents' configurations, and are assigned to the agents by minimizing the total assignment distance using the Hungarian algorithm.

\noindent\textbf{Agent Dynamics:} We use the default dynamics defined in MPE2, where each agent $i$ is modeled as a point mass with damped double-integrator dynamics:
\begin{align}
    p_{t+1}^i &= p_t^i + \dot p_t^i \, \Delta t, \\
    \dot p_{t+1}^i &= (1 - \beta)\,\dot p_t^i +
        \frac{\Delta t}{m}\, u_t^i,
\end{align}
where $u_t^i \in \{(0,0), (\pm 1, 0),(0, \pm 1)\}$ is the control input selected by $a_t^i$, with sampling time $\Delta t = 0.1$, damping coefficient $\beta = 0.25$, and mass $m = 1$.

\noindent\textbf{Training:} We compare our method against four baselines:
\begin{itemize}
    \item \textbf{Normal:} the vanilla decentralized AC-MARL from~\cite[Alg.~2]{zhang2018fully} without attacks;
    \item \textbf{Naive:} the vanilla decentralized AC-MARL from~\cite[Alg.~2]{zhang2018fully} under $F$ Byzantine edge attacks;
    \item \textbf{Projection:} the resilient AC-MARL method from~\cite[Alg.~2]{ye2024resilient} that uses a projection-based defense mechanism. Following the authors' implementation, trimmed-mean aggregation is applied to the hidden-layer parameters before the projection-based updates; and
    \item \textbf{Trimmed-Mean:} the resilient AC-MARL from~\cite[Alg.~1]{wu2021byzantine} that performs an element-wise trimmed-mean operation.
\end{itemize}
 We simulate Byzantine edge attacks by randomly selecting $F$ edges incident to agent $1$, replacing the transmitted messages with parameter tuples obtained by adding a positive offset to each parameter tensor. The offset scaled according to the mean absolute magnitude of the corresponding tensor. Then, it is upper bounded by 1 to ensure numerical stability.

We train all methods for $10000$ episodes using five different random seeds, with each episode consisting of $35$ steps, using 10 critic and reward-function updates per actor update. The learning rates are $\alpha_t^v=\alpha_t^\lambda=0.01$ and $\alpha_t^\theta=0.001$. Updates are performed in batches every 20 episodes. We set the discount factor to $\gamma=0.9$ and each agent selects a random action with probability $\mu=0.1$.

\begin{figure}
    \centering
\includegraphics[width=0.95\linewidth]{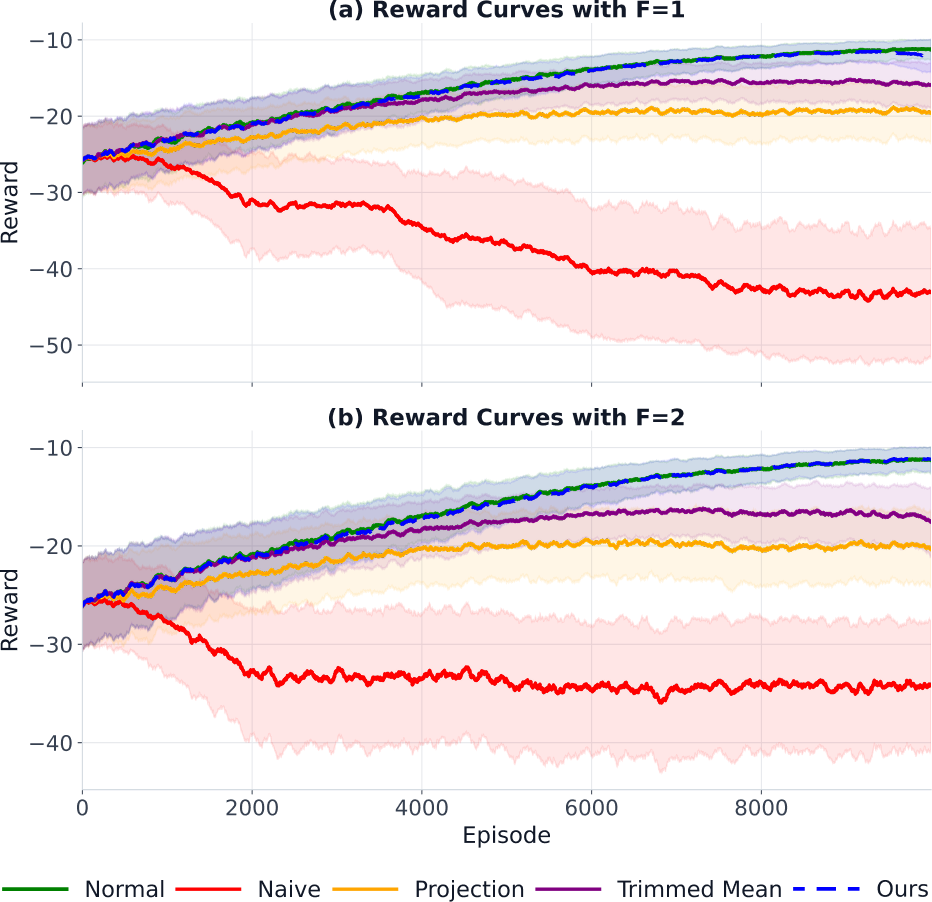}
    \caption{Reward curves under $F$-total Byzantine edge attack, with (a) $F=1$ and (b) $F=2$. Our method attains the same reward level as the attack-free Normal baseline, whereas the other methods converge to suboptimal policies with visibly lower rewards.}
    \label{fig:bounded}
\end{figure}

We consider $F=1$ and $F=2$ with $r=2F+1$. Every $20$ episodes we generate a $(2F+1,0)$-redundant network using the construction mechanism in~\Cref{prop:construction}, while randomly permuting the agent indices to emulate a time-varying communication topology. By construction, the resulting networks satisfy the topological conditions required by Theorems~\ref{thm:critic}-\ref{thm:actor}. By~\Cref{lem:connection}, the network is also $(2F+1)$-robust for all time, satisfying the sufficient topological condition required for the convergence of the projection-based method in~\cite{ye2024resilient}. 

\Cref{fig:bounded} reports the reward curves under $F=1$ and $F=2$ Byzantine edge attacks. Our method matches the attack-free Normal baseline in both settings. This is consistent with Theorems~\ref{thm:critic}-\ref{thm:actor}, which guarantee convergence to equilibria of the limiting ODEs rather than a neighborhood of them. In contrast, the other methods converge to lower reward levels, reflecting the residual errors introduced by their consensus steps under $F$-total Byzantine edge attacks.

\section{Conclusion}
We study resilient actor-critic multi-agent reinforcement learning under Byzantine edge attacks. Our method exploits the redundancy of two-hop communication to decide which messages to trust and filter. We introduce a novel topological condition, $(r,r')$-redundancy, to provide the conditions under which the policy parameters converge almost surely to a locally asymptotically stable equilibrium of the attack-free limiting ODE. We validate our method on a multi-agent formation control task.

\section*{References}
\bibliographystyle{IEEEtran}
\bibliography{references_ll}

\begin{IEEEbiography}[{\includegraphics[width=1in,height=1.25in,clip,keepaspectratio]{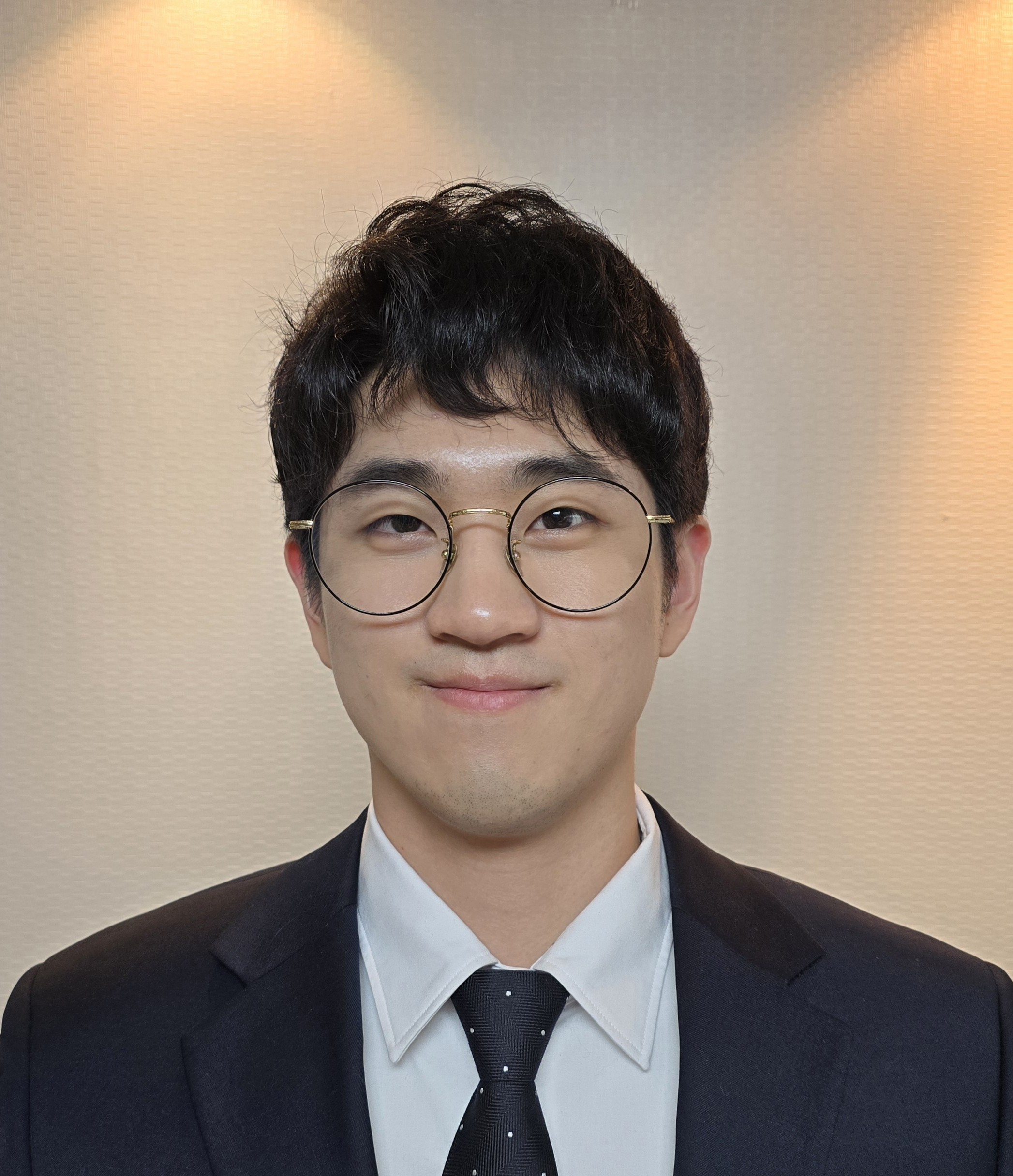}}]{Haejoon Lee} (Student Member, IEEE) received the B.S. degree in applied math and statistics from Stony Brook University, Stony Brook, NY, USA, in 2023. He earned the M.S.
degree in robotics in 2025 from the University of
Michigan, Ann Arbor, MI, USA, where he is currently working toward the Ph.D. degree in robotics, advised by Prof. Dimitra Panagou. 

His research interests include safety, resilience, and security of autonomous systems, with particular emphasis on distributed consensus, optimization, and learning for multi-agent systems in adversarial and uncertain environments. 

\end{IEEEbiography}

\begin{IEEEbiography}[{\includegraphics[width=1in,height=1.25in,trim={0mm 0 0 0},clip,keepaspectratio]{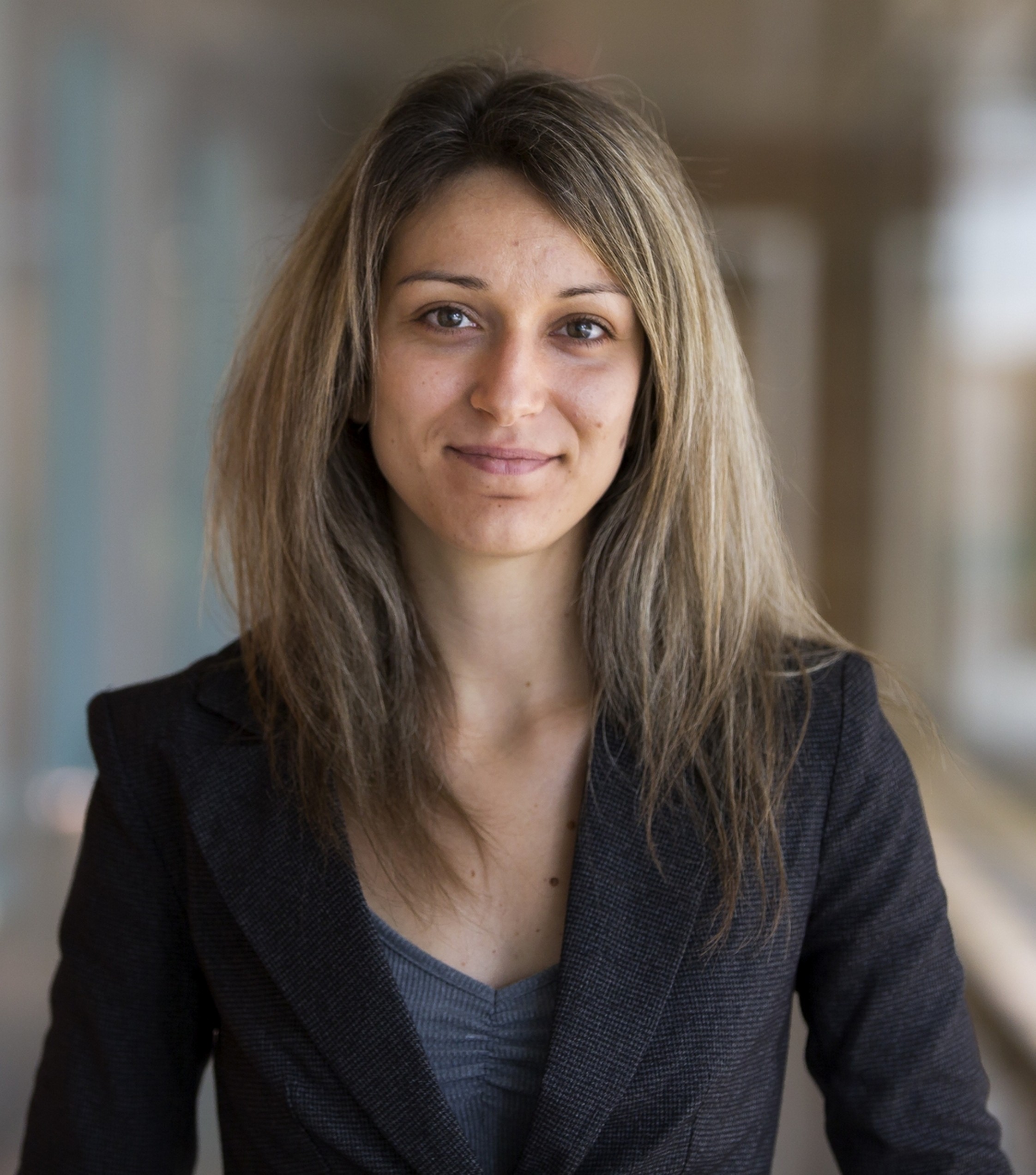}}]{Dimitra Panagou} (Diploma (2006) and PhD (2012) in Mechanical Engineering from the National Technical University of Athens, Greece) is an Associate Professor with the Department of Robotics, with a courtesy appointment with the Department of Aerospace Engineering, University of Michigan. Her research program spans the areas of nonlinear systems and control; multi-agent systems; autonomy; and  aerospace robotics. She is particularly interested in the development of provably-correct methods for the safe and secure (resilient) operation of autonomous systems with applications in robot/sensor networks and multi-vehicle systems under uncertainty. She is a recipient of the NASA Early Career Faculty Award, the AFOSR Young Investigator Award, the NSF CAREER Award, the George J. Huebner, Jr. Research Excellence Award, and a Senior Member of the IEEE and the AIAA.
\end{IEEEbiography}

\end{document}